\documentclass{article}
\usepackage{ijcai22}

\usepackage{times}
\usepackage{soul}
\usepackage{url}
\usepackage[hidelinks]{hyperref}
\usepackage[utf8]{inputenc}
\usepackage[small]{caption}
\usepackage{graphicx}
\usepackage{amsmath}
\usepackage{amsthm}
\usepackage{booktabs}
\usepackage{algorithm}
\usepackage{algorithmic}
\usepackage{subfigure}
\usepackage{multirow}
\usepackage{hhline}
\usepackage{amssymb}
\usepackage{epstopdf}
\usepackage{color}
\usepackage{enumitem,picinpar}
\usepackage{newfloat}
\usepackage{listings}
\usepackage{textcomp,moresize}
\usepackage{bbm}
\usepackage{latexsym}
\usepackage{epsfig}
\usepackage[flushleft]{threeparttable}
\usepackage{amsfonts}

\DeclareMathOperator*{\argmin}{arg\,min}

\newtheorem{theorem}{Theorem}
\newtheorem{lemma}{Lemma}
\newtheorem{definition}{Definition}
\newtheorem{assumption}{Assumption}
\newtheorem{property}{Property}
\newtheorem{remark}{Remark}

\newtheorem{proof*}{Proof}

\def\wrt{\textit{w.r.t.~}}

\newcommand{\E}{\mathbb{E}}

\newcommand{\G}{\mathcal{G}}
\newcommand{\B}{\mathcal{B}}
\newcommand{\PP}{\mathcal{P}}
\newcommand{\F}{\mathcal{F}}

\title{Distributed Dynamic Safe Screening Algorithms for Sparse Regularization}

\author{
Runxue Bao
\and
Xidong Wu\and
Wenhan Xian\And
Heng Huang
\affiliations
Electrical and Computer Engineering Department, University of Pittsburgh, PA, USA
\emails
\{runxue.bao, xidong\_wu, wex37, heng.huang\}@pitt.edu
}

\begin{document}

\maketitle

\begin{abstract}
Distributed optimization has been widely used as one of the most efficient approaches for model training with massive samples. However, large-scale learning problems with both massive samples and high-dimensional features widely exist in the era of big data. Safe screening is a popular technique to speed up high-dimensional models by discarding the inactive features with zero coefficients. Nevertheless, existing safe screening methods are limited to the sequential setting.  In this paper, we propose a new distributed dynamic safe screening (DDSS) method for sparsity regularized models and apply it on shared-memory and distributed-memory architecture respectively, which can achieve significant speedup without any loss of accuracy by simultaneously enjoying the sparsity of the model and dataset. To the best of our knowledge, this is the first work of distributed safe dynamic screening method.  Theoretically, we prove that the proposed method achieves the linear convergence rate with lower overall complexity and can eliminate almost all the inactive features in a finite number of iterations almost surely. Finally, extensive experimental results on benchmark datasets confirm the superiority of our proposed method.
\end{abstract}

\begin{table*}[t]
\renewcommand\arraystretch{1.1}
	\center
	\caption{Summary of dynamic safe  screening methods. ``Model'' refers to the model it can solve where ``MTL \& MLR'' means multi-task Lasso and multinomial logistic regression. ``Safe'' represents there are no active features eliminated. ``Generalized'' means whether it is limited to a specific model. ``Distributed'' represents whether it can work on  distributed-memory architecture. ``Scalablity'' represents whether it is scalable with sample size $n$.}
  \resizebox{\textwidth}{!}{
	\begin{tabular}{c|c|c|c|c|c}
		\hline
		{\textbf{Reference}}  &  {\textbf{Model}}  &  {\textbf{Safe}}
	 &  {\textbf{Generalized}}  &  {\textbf{Distributed}}    
	  &  {\textbf{Scalablity}} \\
		\hline
		 \cite{fercoq2015mind} &  Lasso  & Yes & No & No  & No \\
		\cite{ndiaye2016gap} &  Sparse-Group Lasso   & Yes & No  & No & No \\\
	   \cite{shibagaki2016simultaneous}   & Sparse SVM & Yes  & No  & No & No  \\ 
		\cite{ndiaye2017gap} &  MTL \& MLR   & Yes & No & No  & No \\
		 \cite{rakotomamonjy2019screening}   & Proximal Weighted Lasso & Yes & No  & No & No  \\
	 \cite{bao2020fast}   & OWL Regression & Yes & No  & No & No  \\
		\hline
		 DDSS (Ours)   & Problem (\ref{eq:general})  & Yes & Yes & Yes  &  Yes  \\
		\hline
	\end{tabular}}
\label{table:method1}
\end{table*}

\begin{table*}[t]
\renewcommand\arraystretch{1.1}
	\center
\caption{Comparison between PSE and our DDSS algorithm.  ``Data Sparsity'' represents whether it can benefit from the sparsity of data.}
 \resizebox{\textwidth}{!}{
	\begin{tabular}{c|c|c|c|c|c|c}
		\hline
	{\textbf{Reference}}  &  {\textbf{Model}}  &  {\textbf{Dynamic}}
	 &  {\textbf{Safe}} &  {\textbf{Distributed}}   & {\textbf{Scalablity}}  &  {\textbf{Data Sparsity}}  \\
		\hline
PSE \cite{li2016parallel} &  Lasso  & No & No & No  & No  & No    \\
		\hline
		 DDSS (Ours)   & Problem (\ref{eq:general})  & Yes & Yes & Yes &  Yes & Yes \\
		\hline
	\end{tabular}}
	\label{table:method2}
\end{table*}

\section{INTRODUCTION}
Learning sparse representations plays a important role in many  machine learning and signal processing applications \cite{lustig2008compressed,shevade2003simple,wright2009sparse,wright2010sparse,chen2021learnable,bian2021optimization,zhang2021ddn2}. In the past decades, many models with sparse regularization have achieved great successes in high-dimensional scenarios by encouraging the model sparsity \cite{tibshirani1996regression,ng2004feature,yuan2006model,bao2019efficient}. Let $A=\left[a_{1}, \cdots, a_{n}\right]^{\top} \in \Re^{n \times p}$, in this paper, we consider the following composite optimization problem: 
\begin{eqnarray} 
\min\limits_{x \in \Re^{p} } \PP(x):=  \F(x) +\lambda\Omega(x).
\label{eq:general}
\end{eqnarray}
where $x$ is the model coefficient, $\Omega(x)$ is the block-separable sparsity-inducing norm,  $\F(x)  =  \frac{1}{n} \sum\nolimits_{i=1}^{n}f_i(a_i^\top x)$ is the data-fitting loss,  and $\lambda$ is the regularization parameter. We denote $\F_i(x) = f_i(a_i^\top x)$ for simplicity. Given partition $\G$ of the coefficients, we denote the sub-matrix of $A$ with the columns of $\G_j$ as $A_{j} \in \Re^{n \times |\G_j|}$ and have $\Omega(x) = \sum\nolimits_{j=1}^{q} \Omega_j(x_{\G_j})$.

Distributed learning has been actively studied in machine learning community due to its capability for tackling big data computations and numerous large-scale applications. In literature, many distributed learning  methods have been proposed to accelerate different optimization algorithms on shared-memory architecture \cite{langford2009slow,recht2011hogwild,reddi2015variance,zhao2016fast,mania2017perturbed,leblond2017asaga,meng2017asynchronous,pedregosa2017breaking,zhou2018simple} and distributed-memory architecture \cite{agarwal2012distributed,dean2012large,zhang2014asynchronous,lian2015asynchronous,zhang2016asynchronous}. For smooth
optimization problems, asynchronous stochastic methods, e.g., Hogwild! \cite{recht2011hogwild}, Lock-Free SVRG \cite{reddi2015variance}, AsySVRG \cite{zhao2016fast}, KroMagnon \cite{mania2017perturbed}, and ASAGA \cite{leblond2017asaga} were proposed. Furthermore, AsyProxSVRG \cite{meng2017asynchronous}, AsyProxSBCDVR \cite{pmlr-v84-gu18a}, ProxASAGA \cite{pedregosa2017breaking}, and AsyMiG \cite{zhou2018simple} were proposed to solve non-smooth composite optimization problems. However, large-scale problems with both massive samples and high-dimensional features widely exist in the era of big data, which still suffer huge burden for computation and memory costs.

By exploiting the sparsity, safe screening is an effective method to accelerate high-dimensional sparse models by pre-identifying inactive features and thereby avoiding the useless computation for training. Safe static screening \cite{Laurent2012safe} was first proposed for $l_{1}$ regularized problems, which only performs once prior to the optimization. The strong screening method \cite{tibshirani2012strong} was proposed for Lasso via heuristic strategies and could discard features wrongly. Furthermore, \cite{wang2013lasso} proposed the sequential method for Lasso, which needs to estimate the exact dual solution and could be unsafe in practice. Recently, an effective dynamic method \cite{fercoq2015mind} was proposed for Lasso with good performance. Due to its better empirical and theoretical results, many dynamic rules \cite{shibagaki2016simultaneous,ndiaye2015gap,ndiaye2016gap,ndiaye2017gap,rakotomamonjy2019screening} were proposed for a broad class of sparse models with separable penalties. Recently, \cite{bao2020fast} introduces a dynamic screening rule to effectively handle the inseparable OWL regression \cite{bao2019efficient} via an iterative strategy. Table \ref{table:method1} summarizes existing representative safe dynamic screening methods. However, all of these methods are limited to sequential (single-machine) setting and thus huge requirement of computational complexity hinders its application on large-scale datasets.  Therefore, distributed safe dynamic screening algorithms are promising and sorely needed to accelerate existing training methods for large-scale problems.

\iffalse
Sparsity regularized models have had a great impact on machine learning community in the past decades, such as \cite{tibshirani1996regression,ng2004feature,yuan2006model}. %The most popular cases include  Lasso \cite{tibshirani1996regression}, sparse logistic regression \cite{ng2004feature}, group Lasso \cite{yuan2006model}, and etc.   
Generally, these models aim to solve the composite optimization problem with a data-fitting loss $\F(x) = \frac{1}{n} \sum\nolimits_{i=1}^{n}f_i(a_i^\top x)$ with $L$-Lipschitz gradient plus a block-separable norm $\Omega(x) = \sum\nolimits_{j=1}^{q} \Omega_j(x_{\G_j})$ as:
\begin{eqnarray} 
\min\limits_{x \in \Re^{p} } \PP(x):=  \F(x) +\lambda\Omega(x),
\label{eq:general}
\end{eqnarray}
where $x \in \Re^{p}$ is the model coefficient that we need to learn, $\G$ is a partition of the coefficients, and $\lambda$ is the regularization parameter. We have matrix $A=\left[a_{1}, \cdots, a_{n}\right]^{\top} \in \Re^{n \times p}$ has $n$ samples and $p$ features and  $A_{j} \in \Re^{n \times |\G_j|}$ denotes the sub-matrix of $A$ with the columns of $\G_j$. We assume $n$ and $p$ could be huge in the era of big data. 
\fi

In this paper, we propose a new distributed dynamic screening method to solve sparse models and apply it on shared-memory  (Sha-DDSS) and distributed-memory (Dis-DDSS) architecture, which can accelerate the training process significantly without any loss of accuracy. Specifically, DDSS not only conducts the screening to eliminate inactive features to enjoy the sparsity of the model, but also conducts sparse proximal gradient update with the nonzero partial gradients to enjoy the sparsity of the dataset. To further accelerate DDSS, we also reduce the gradient variance over the active set and thus we can utilize a constant step size to achieve the linear convergence. On distributed-memory system, to reduce the computational burden of the server node, we utilize a decouple strategy to improve the scalability of DDSS \wrt the number of workers.  

\paragraph{Contributions.}
The main contributions of our work can be summarized as follows. \textbf{First}, we propose a new distributed dynamic safe screening framework for generalized sparse models, which is easy-to-implement on the both shared-memory and distributed-memory architecture. To the best of knowledge, this is the first work of distributed dynamic safe screening.
\textbf{Second}, we rigorously prove the proposed DDSS method can achieve linear convergence rate $O(\log(1/\epsilon))$, reduce the per-iteration cost from $O(p)$ to $O(r)$ where $r \ll p$, and finally achieve a lower overall computational complexity under the strongly convex condition. \textbf{Third}, we prove almost sure finite time identification of the active set to confirm the effectiveness of our DDSS method.
Finally, we empirically show that our proposed method can achieve significant acceleration and linear speedup property.

\subsection{Related Works}

\paragraph{Parallel Static Screening}
Parallel static screening (PSE) \cite{li2016parallel} is most related work to ours. Table \ref{table:method2} summarizes the advantages of our DDSS method over PSE. First, DDSS is dynamic and conducted during the whole training process. Thus, DDSS can accelerate and meanwhile benefit from the convergence of the optimization algorithm. Second, our DDSS is safe for the training and can guarantee the model accuracy. Third,  DDSS is stochastic and can scale well on both samples and features. Fourth, DDSS can enjoy the data sparsity by performing sparse proximal updates to further accelerate the training. Lastly, DDSS can solve Problem (\ref{eq:general}).

\paragraph{Asynchronous Stochastic Method}
An asynchronous doubly stochastic method was proposed in \cite{meng2017asynchronous}, which performs coordinate updates without considering any sparsity and thus can be very slow. In \cite{pmlr-v84-gu18a}, they proposed a asynchronous doubly stochastic method for group regularized learning problems. Moreover, \cite{leblond2017asaga} proposed an asynchronous sparse incremental gradient method, which does not enjoy the sparsity of the model and cannot be easily incorporated into the dynamic screening method to accelerate the training. However, our DDSS can simultaneously enjoy the sparsity of the model and the dataset.

\section{PROPOSED METHOD}
We first propose the distributed dynamic safe screening (DDSS) method and then apply it to shared-memory and distributed-memory architecture respectively.

\subsection{Distributed Dynamic Safe Screening}
\begin{algorithm}[H]
\renewcommand{\algorithmicrequire}{\textbf{Input:}}
\renewcommand{\algorithmicensure}{\textbf{Output:}}
\caption{Sha-DDSS-Naive}
\begin{algorithmic}[1]
\STATE {\bfseries Input:} $x^0_{\B_{0}} \in \Re^p$, step size $\eta$,  inner loops $K$
\FOR{$s=0$ {\bfseries to} $S-1$}
\STATE All threads parallelly compute $\nabla \F(x^0_{\B_{s}})$
\STATE Compute $y^s$ by (\ref{eq:scaling}) and $\B_{s+1}$ from $\B_{s}$ by (\ref{eq:screening})
\STATE Update $A_{\B_{s+1}}, x^0_{\B_{s+1}}$
\STATE For each thread, do:
\FOR{$t=0$ {\bfseries to} $K-1$}
\STATE Read $\hat{x}^t_{\B_{s+1}}$ from the shared memory
\STATE Randomly sample $i$ from $\{1,2,\ldots, n\}$
\STATE $v^{s}_t=\nabla f_{i}(a_{i,\B_{s+1}}^\top\hat{x}^t_{{\B_{s+1}}}) $
\STATE Update $x_{\B_{s+1}}^{t+1} =\operatorname{prox}_{\eta \lambda \Omega}(\hat{x}^{t}_{{\B_{s+1}}}-\eta v^{s}_{t})$
\ENDFOR
\STATE $x_{\B_{s+1}} = x_{{\B_{s+1}}}^{K}, x^0_{\B_{s+1}} = x_{\B_{s+1}}$
\ENDFOR
%\ENSURE Coefficient $x_{\B_{s+1}}^{s+1}$.
\end{algorithmic}
\label{algorithm:Sha-DDSS-Naive}
\end{algorithm}

To enjoy the sparsity of the model coefficients, we first give a naive implementation of DDSS on shared-memory architecture in Alg. \ref{algorithm:Sha-DDSS-Naive} and on distributed-memory architecture in Alg. \ref{algorithm:Dis-DDSS-NaiveServer} and  \ref{algorithm:Dis-DDSS-NaiveWorker} respectively. We mainly take the algorithm on shared-memory architecture as the example to illustrate our idea. During the training, DDSS only need solve a sub-problem with constantly decreasing size by discarding useless features.

Specifically, Alg. \ref{algorithm:Sha-DDSS-Naive} has two loops. We denote the original problem as $\PP_0$ and the full set as $\B_0$.  At the $s$-th iteration of the outer loop,  we denote the active set  as $\B_s$ and suppose $\B_s$ has $q_s$ active blocks with total $p_s$ active features. Thus, sub-problem $\PP_{s}$ is over set  $\B_s$. Then we can compute the dual $y^s$ as
\begin{eqnarray} 
y^s= - \nabla \mathcal{F}( x^0_{\B_{s}}) / \max (1, \Omega^{D}(A^{\top}_{\B_{s}} \nabla \mathcal{F}( x^0_{\B_{s}}))/\lambda ),
\label{eq:scaling}
\end{eqnarray}
where dual norm  $\Omega^{D}(u)=\max_{\Omega(v) \leq 1}\langle v, u\rangle $. With the obtained dual variable, we can eliminate inactive feature blocks (see details in \cite{ndiaye2017gap,bao2020fast}) for $\forall j \in \B_{s}$ as
\begin{eqnarray} 
 \Omega^D_j(A^\top_{j} y^s) + \Omega^D_j(A_{j}) \sqrt{2L( \PP(x^0_{\B_{s}})-D(y^s))} < n \lambda,
\label{eq:screening}
\end{eqnarray}
to update $\B_{s+1}$ where $L$ is the Lipschitz constant. The dual objective $D(y^s)$ of $\PP_{s}$ can be computed as:
\begin{eqnarray} 
 \max\limits_{y^s}     D(y^s):= -\frac{1}{n}\sum\nolimits_{i=1}^{n} f_{i}^{*}(-y_{i}^s), \nonumber \\
s.t. \quad \Omega_{j}^{D}(A_{j}^{\top} y^s) \leq n \lambda, \quad \forall j \in {1,\ldots,q_s}.
\label{eq:dual}
\end{eqnarray}

For the inner loop, all the updates are conducted on $\B_{s+1}$. First, each thread inconsistently read $\hat{x}^{t}_{{\B_{s+1}}}$ from the shared memory and randomly choose sample $i$ to compute the stochastic gradient on $\B_{s+1}$. Then, the proximal step is conducted with the stochastic gradient. By (\ref{eq:screening}), we can constantly reduce the model size and the parameter size to accelerate the training by exploiting the sparsity of the model. Since each variable $x_i$ discarded by the screening  must be zero for the optimum solution, this method is safe for the training.

\begin{algorithm}[H]
\renewcommand{\algorithmicrequire}{\textbf{Input:}}
\renewcommand{\algorithmicensure}{\textbf{Output:}}
\caption{Dis-DDSS-Naive (Server Node)}
\begin{algorithmic}[1]
\FOR{$s=0$ {\bfseries to} $S-1$}
\STATE flag = True
\STATE Broadcast flag and $x^0_{\B_{s}}$ to all workers
\STATE Receive gradients from all workers 
\STATE $\nabla \F(x^0_{\B_{s}}) = \frac{1}{n}\sum\nolimits_{k=1}^{l} \nabla \F^k (x^0_{\B_{s}})$
\STATE Compute $y^s$ by (\ref{eq:scaling})
\STATE Update $\B_{s+1} \subseteq \B_{s}$ by (\ref{eq:screening})
\STATE Broadcast $\B_{s+1}$ and $\nabla \F(x^0_{\B_{s}})$ to all workers
\STATE flag = False
\STATE Broadcast flag to all workers
\FOR{$t=0$ {\bfseries to} $K-1$}
\STATE Receive $v^{s}_{t}$ from worker
\STATE $x^{t+1}_{\B_{s+1}} = \operatorname{prox}_{\eta \lambda \Omega}(x^{t}_{\B_{s+1}}-\eta v_{t})$
\ENDFOR
\STATE $x_{\B_{s+1}} = x^{K}_{\B_{s+1}},  x^0_{\B_{s+1}} = x_{\B_{s+1}}$
\ENDFOR
%\ENSURE Coefficient $x^{s+1}_{\B_{s+1}}$.
\end{algorithmic}
\label{algorithm:Dis-DDSS-NaiveServer}
\end{algorithm}

\begin{algorithm}[H]
\renewcommand{\algorithmicrequire}{\textbf{Input:}}
\renewcommand{\algorithmicensure}{\textbf{Output:}}
\caption{Dis-DDSS-Naive (Worker Node $k$)}
\begin{algorithmic}[1]
\IF{flag = True}
\STATE Receive  $x^{0}_{\B_{s}}$ from server
\STATE Compute and send gradient $ \nabla \F^k(x^0_{\B_{s}}) = \sum\nolimits_{i\in n_k} \nabla  f_i(a_{i,\B_{s}}^\top x^0_{\B_{s}})$
\STATE Receive $ \B_{s+1}$ from server
\STATE Update $A_{i\in n_k, \B_{s+1}},x^0_{\B_{s+1}}$
\ELSE
\STATE Receive  $x^{d(t)}_{\B_{s+1}}$ from server
\STATE Randomly sample $i$ from  $\{1,2,\ldots, n_k\}$
\STATE Compute 
$v_{t}^{s}=\nabla f_{i}(a_{i,\B_{s+1}}^\top x^{d(t)}_{\B_{s+1}}) $ 
\STATE Send $v^{s}_{t}$  to server
\ENDIF
\end{algorithmic}
\label{algorithm:Dis-DDSS-NaiveWorker}
\end{algorithm}

\paragraph{Variance Reduction on the Active Set} 
However, the variance of the gradient estimation in Alg. \ref{algorithm:Sha-DDSS-Naive} caused by stochastic sampling does not converge to zero. Thus, we have to use a diminishing step size and can only obtain very little progress for each update. Thus,
Alg. \ref{algorithm:Sha-DDSS-Naive} can only attain a sublinear convergence rate even when $\PP$ is strongly convex.

Since the full gradient has been computed by the outer loop for the  elimination step, inspired by the variance-reduced technique in \cite{xiao2014proximal,li2021fully}, we can adjust the gradient estimation over $\B_{s+1}$ with the exact gradient  from the outer loop  without additional computational costs as 
\begin{eqnarray} 
v^{s}_{t}&=&\nabla f_{i}(a_{i,\B_{s+1}}^\top \hat{x}^t_{\B_{s+1}})-\nabla f_{i}(a_{i,\B_{s+1}}^\top x^0_{\B_{s+1}})
\nonumber \\
&& + \nabla \F( x^0_{\B_{s+1}}),
\end{eqnarray} 
which can guarantee that the  variance of stochastic gradients asymptotically converges to zero. Therefore, we can use a constant step size to achieve more progress for each iteration and finally achieve a linear convergence rate for strongly convex function.

\paragraph{Sparse Proximal Gradient Update}  In practice, sparsity widely exists in large-scale datasets. To utilize the sparsity of the dataset, we only need to update the blocks that contains nonzero partial gradients. Thus, some blocks might be updated for more times while others for less times. Inspired by \cite{leblond2017asaga} for Proximal SAGA, we define a block-wise reweighting matrix to make a weighted projection on the blocks.   Specifically, we define $\Psi_i$ as the set of blocks that intersect the nonzero coefficients of $\nabla f_i$. Let $n_{\G}$ be the number of occurrences that $\G \in \Psi_{i}$, if $n_{\G}>0$, we define  $d_{\G}=n/n_{\G}$. Otherwise, we can ignore that block directly.  Thus, we can define diagonal matrix  $\left[D_{i}\right]_{\G, \G}=d_{\G} I_{\left|\G\right|}$ for each block $i$ and thus the gradient over set $\B_{s+1}$ can be formulated as 
\begin{eqnarray} 
v^{s}_{t} &=& \nabla f_{i}(a_{i,\B_{s+1}}^\top \hat{x}^t_{ \B_{s+1}})-\nabla f_{i}(a_{i,\B_{s+1}}^\top x^0_{\B_{s+1}})
\nonumber \\ &&+ D_{i, {\B_{s+1}}} \nabla \F( x^0_{\B_{s+1}}).
\label{eq:vr_gradient}
\end{eqnarray} 
Thus, we only need compute a sparse gradient and conduct a sparse update over the active set and the computational cost is further reduced.

On the other hand, the proximal operator of original Problem (\ref{eq:general}) is computed  as
\begin{eqnarray} 
\operatorname{prox}_{\eta \lambda \Omega}\left(x^{\prime}\right)=\argmin _{x \in \Re^{p}} \frac{1}{2 \eta}\left\|x-x^{\prime}\right\|^{2}+\lambda \Omega(x),
\end{eqnarray} 
which needs to update all the coordinates for each iteration. Considering the sparsity of the dataset again, we only need to update the blocks that contains nonzero partial gradients. Thus, based on the reweighting matrix $D$, we use a block-wise weighted norm $\phi_{i}(x)=\sum_{\G \in \Psi_{i}} d_{\G} \Omega_{\G}(x)$ to displace $\Omega(x)$. Note it is easy to verify that $\E \phi_{i}(x)=\Omega(x)$. Thus, the new sparse proximal operator can be computed as
\begin{eqnarray} 
\operatorname{prox}_{\eta \lambda \phi_{i}}\left(x^{\prime}\right)=\argmin _{x \in \Re^{p}} \frac{1}{2 \eta}\left\|x-x^{\prime}\right\|^{2}+\lambda \phi_{i}(x).
\end{eqnarray} 
Since we only update the blocks in $\Psi_i$, which could be much less than the one needs a full pass of $p$ coordinates due to the sparsity, we can save much computation and memory cost here. To sum it up, we can conduct the sparse gradient update and sparse proximal operator to accelerate the training by enjoying the sparsity of the dataset.

\paragraph{Decoupled Proximal Update}
On distributed-memory architecture, multiple workers compute the gradients and send them to the server. The server computes the proximal operator. When the proximal step is time-consuming, doing this in the server would be the computational bottleneck of the whole algorithm. \cite{li2016make} proposed a decoupled method to off-load the computational task of the proximal step to workers. Thus, the server only does simple addition computation, which can achieve a sublinear converge rate and perform better than the coupled method. 

To relieve the computation cost of the server in our algorithm, the proximal mapping step is computed by workers and the server only needs to do the element-wise computation. The workers conduct the proximal operator as
\begin{eqnarray} 
x_{\B_{s+1}}^{t+1} =\operatorname{prox}_{\eta \lambda \phi_{i}}(x^{t}_{{\B_{s+1}}}-\eta v^{s}_{t}),
\end{eqnarray} 
and send the difference 
\begin{eqnarray} 
\delta^s_{t} = \operatorname{prox}_{\eta \lambda \phi_{i}}(x^{d(t)}_{\B_{s+1}}-\eta v^{s}_{t})-x^{d(t)}_{\B_{s+1}},
\end{eqnarray} 
between the parameter $x^{d(t)}_{\B_{s+1}}$ and the output of the proximal operator to the server. Therefore, the server only does simple addition computation, which makes the algorithm suitable to parallelize to achieve linear speedup property and can be accelerated via increasing the number of workers. 

\subsection{DDSS on Shared-Memory Architecture}
\begin{algorithm}[H]
\renewcommand{\algorithmicrequire}{\textbf{Input:}}
\renewcommand{\algorithmicensure}{\textbf{Output:}}
\caption{Sha-DDSS}
\begin{algorithmic}[1]
\FOR{$s=0$ {\bfseries to} $S-1$}
\STATE All threads parallelly compute $\nabla \F(x_{\B_{s}}^{0})$
\STATE Compute  $y^s$ by (\ref{eq:scaling}) and $\B_{s+1}$ from $\B_{s}$ by (\ref{eq:screening})
\STATE Update $A_{\B_{s+1}},x^0_{\B_{s+1}},\nabla \F(x^0_{\B_{s+1}})$
\STATE For each thread, do:
\FOR{$t=0$ {\bfseries to} $K-1$}
\STATE Read $\hat{x}^{t}_{\B_{s+1}}$ from the shared memory
\STATE Randomly sample $i$ from $\{1,2,\ldots, n\}$
\STATE Compute $v^{s}_{t}$ by (\ref{eq:vr_gradient})
\STATE $\delta^s_{t} = \operatorname{prox}_{\eta \lambda \phi_{i}}(\hat{x}^{t}_{\B_{s+1}}-\eta v^{s}_{t})-\hat{x}^t_{\B_{s+1}}$
\STATE $x_{\B_{s+1}}^{t+1} = x_{\B_{s+1}}^{t} + \delta^s_{t}$
\ENDFOR
\STATE $x_{\B_{s+1}} = x^K_{\B_{s+1}}, x^0_{\B_{s+1}} = x_{\B_{s+1}}$
\ENDFOR
%\ENSURE Coefficient $x_{\B_{s+1}}^{s+1}$.
\end{algorithmic}
\label{algorithm:Sha-DDSS}
\end{algorithm}

On shared-memory architecture, our Sha-DDSS algorithm is in Alg. \ref{algorithm:Sha-DDSS}. Suppose we have $l$ cores, in the outer loop, all threads parallelly compute $\nabla \F(x_{\B_{s}}^{0})$ and $y^s$, and perform the elimination. With new set $\B_{s+1}$, we update $A_{\B_{s+1}}$, $x^0_{\B_{s+1}}$, and $\nabla \F(x^0_{\B_{s+1}})$. In the inner loop, the algorithm optimizes over  $\B_{s+1}$. Multiple threads update the parameter asynchronously, which means the parameter can be read and written without locks. Specifically, each thread inconsistently read $\hat{x}^{t}_{\B_{s+1}}$ from the shared memory and then choose sample $i$ from $\{1,2,\ldots,n\}$. As (\ref{eq:vr_gradient}), we compute the gradient $v^{s}_{t}$ over $\B_{s+1}$. Then we conduct the proximal step, compute the update $\delta^s_{t}$, and add it to the shared memory.

Notably, first, at the $s$-th iteration, by exploiting the sparsity of the model, our Alg. \ref{algorithm:Sha-DDSS} only solve sub-problem $\PP_{s+1}$ over $\B_{s+1}$, which is much more efficient than training of the full model. Thus, the full gradients at the $s$-th iteration in our algorithm is only computed with $p_{s}$ coefficients, which is much less than $O(p)$ in practice. Second, by exploiting the sparsity of the dataset, we only conduct sparse gradient update and sparse proximal update, which is very efficient for large-scale real-world datasets. Third, by reducing the gradient variance, we can use a constant step size to improve the convergence and finally achieve a linear convergence rate for strongly convex function. Lastly, all the threads works asynchronously, which is very efficient and easy to parallelize. Hence, the computation and memory costs of large-scale training can be effectively reduced.

\subsection{DDSS on Distributed-Memory Architecture}

On distributed-memory architecture, our Dis-DDSS algorithm is summarized in Alg. \ref{algorithm:Dis-DDSSServer} and \ref{algorithm:Dis-DDSSWorker}. In Dis-DDSS, suppose we have one server node and $l$ local worker nodes where each worker stores $n_k$ samples. When the flag is True, in the outer loop of the server node, the server broadcasts the flag and $x^{0}_{\B_{s}}$ to the workers.  At the worker node, worker $k$ receives $x^{0}_{\B_{s}}$ from the server, computes the gradient over $n_k$ samples, and then sends it to the server node. With the gradients received from all the workers, the server node computes the full gradients and send them to the workers. By (\ref{eq:scaling}), the server node computes $y^s$ and then performs the elimination to obtain $\B_{s+1}$ and then send it to the workers. The worker node receives $ \nabla \F(x^{0}_{\B_{s}})$ and  $\B_{s+1}$ from the server. With $\B_{s+1}$, the worker updates  $A_{i\in n_k, \B_{s+1}}, x^{0}_{\B_{s+1}}$ and $ \nabla \F(x^{0}_{\B_{s+1}})$.

\begin{algorithm}[H]
\renewcommand{\algorithmicrequire}{\textbf{Input:}}
\renewcommand{\algorithmicensure}{\textbf{Output:}}
\caption{Dis-DDSS (Server Node)}
\begin{algorithmic}[1]
\FOR{$s=0$ {\bfseries to} $S-1$}
\STATE flag = True
\STATE Broadcast flag and $x^{0}_{\B_{s}}$ to all workers
\STATE Receive gradients from all workers 
\STATE $\nabla \F(x^{0}_{\B_{s}}) = \frac{1}{n}\sum\nolimits_{k=1}^{l} \nabla \F^k(x^{0}_{\B_{s}})$
\STATE Compute $y^s$ by (\ref{eq:scaling})
\STATE Update $\B_{s+1} \subseteq \B_{s}$ by (\ref{eq:screening})
\STATE Broadcast $\B_{s+1}$ and $\nabla \F(x^{0}_{\B_{s}})$ to all workers
\STATE flag = False
\STATE Broadcast flag to all workers
\FOR{$t=0$ {\bfseries to} $K-1$}
\STATE Receive $\delta^s_{t}$ from one worker
\STATE Update  $x^{t+1}_{\B_{s+1}} = x^{t}_{\B_{s+1}} + \delta^s_{t}$
\ENDFOR
\STATE $x_{\B_{s+1}} = x^{K}_{\B_{s+1}}, x^0_{\B_{s+1}} = x_{\B_{s+1}}$
\ENDFOR
%\ENSURE Coefficient $x^{s+1}_{\B_{s+1}}$.
\end{algorithmic}
\label{algorithm:Dis-DDSSServer}
\end{algorithm}

\begin{algorithm}[H]
\renewcommand{\algorithmicrequire}{\textbf{Input:}}
\renewcommand{\algorithmicensure}{\textbf{Output:}}
\caption{Dis-DDSS (Worker Node $k$) }
\begin{algorithmic}[1]
\IF{flag = True}
\STATE Receive  $x^{0}_{\B_{s}}$ from server
\STATE Compute and send gradient $ \nabla \F^k(x^{0}_{\B_{s}}) = \sum\nolimits_{i\in n_k} \nabla  f_i(a_{i,\B_{s}}^\top x^{0}_{\B_{s}})$
\STATE Receive $ \nabla \F(x^{0}_{\B_{s}})$ from server
\STATE Receive $ \B_{s+1}$ from server
\STATE Update $A_{i\in n_k, \B_{s+1}}, x^{0}_{\B_{s+1}}, \nabla \F(x^{0}_{\B_{s+1}})$
\ELSE
\STATE Receive  $x^{d(t)}_{\B_{s+1}}$ from server
\STATE Randomly sample $i$ from  $\{1,2,\ldots, n_k\}$
\STATE Compute  
$v_{t}^{s}=\nabla f_{i}(a_{i,\B_{s+1}}^\top x^{d(t)}_{\B_{s+1}}) - \nabla f_{i}(a_{i,\B_{s+1}}^\top x^{0}_{\B_{s+1}}) + D_{i, {\B_{s+1}}} \nabla \F(x^0_{\B_{s+1}})  $ 
\STATE Send $\delta_t = \operatorname{prox}_{\eta \lambda \phi_{i}}(x^{d(t)}_{\B_{s+1}}-\eta v^{s}_{t}) -x^{d(t)}_{\B_{s+1}} $  to server
\ENDIF
\end{algorithmic}
\label{algorithm:Dis-DDSSWorker}
\end{algorithm}

When the flag is False, the server broadcasts the flag  to the workers. At the workers, the algorithm optimizes over $\B_{s+1}$. Multiple workers update the parameter asynchronously. Specifically, each worker receives stale parameter $x^{d(t)}_{\B_{s+1}}$ from the server.  The worker first chooses sample $i$ from $\{1,2,\ldots,n_k\}$ and compute $v^{s}_{t}$ over $\B_{s+1}$ as (\ref{eq:vr_gradient}). Following the decoupling strategy, we compute the proximal step and the update at the worker node. Finally, the worker send it to server. In the inner loop of the server node, with the update information $\delta^s_{t}$ from workers,  $x^{t+1}_{\B_{s+1}}$ is updated by $\delta^s_{t}$ via only simple addition computation. 

First, similar to Alg. \ref{algorithm:Sha-DDSS}, our Alg.  \ref{algorithm:Dis-DDSSServer} and \ref{algorithm:Dis-DDSSWorker} is very computationally efficient by exploiting the sparsity of the model and the dataset and reducing the gradient variance. Second,
the time-consuming proximal step from the server is off-loaded to all the workers, which is very efficient and easy to parallelize. Third, by the eliminating and the sparse update, the communication cost between the server and the workers is also much less than the full updates. Overall, the computation, memory and communication costs could be effectively reduced during the training.

\section{THEORETICAL ANALYSIS}
In this section, we provide the rigorous theoretical analysis on the convergence and screening ability for DDSS on shared-memory architecture. Note our analysis can be easily extended to distributed-memory architecture. All the proof is provided in the appendix.

\subsection{Assumptions, Definitions, and Properties}

\begin{assumption} [Strong Convexity]
$\Omega(x)$ is convex and block separable. $\F(x)$ is $\mu$-strongly convex, i.e., $\forall x, x^{\prime} \in \Re^{p},
$ we have $\F(x^{\prime}) \geq \F(x)+\nabla \F(x)^{\top}(x^{\prime}-x)+\frac{\mu}{2}\|x^{\prime}-x\|^{2}.$
\label{assumption:strongconvexity}
\end{assumption}

\begin{assumption} [Lipschitz Smooth]
Each $\F_i(x)$ is differentiable and Lipschitz gradient continuous with $L$, i.e., $\exists L>0,$ such that $\forall x, x^{\prime} \in \Re^{p}$, we have $\left\|\nabla \F_i(x)-\nabla \F_i(x^{\prime})\right\|  \leq L\|x-x^{\prime}\|.$
\label{assumption:lipschitz}
\end{assumption}

\begin{assumption}[Bounded Overlapping]
There exists a  bound $\tau$ on the number of iterations that overlap. The bound $\tau$ means each writing at  iteration $t$ is guaranteed to successfully performed into the memory before iteration $t+\tau + 1$.
\label{assumption:overlapping}
\end{assumption}

\begin{remark}
Asm. \ref{assumption:strongconvexity} implies $\PP(x)$ is also $\mu$-strongly convex. Asm. \ref{assumption:lipschitz} implies that $\F(x)$ is also Lipschitz gradient continuous. We denote $\kappa := L/\mu$ as the condition number. Asm. \ref{assumption:overlapping} means the delay that asynchrony
may cause is upper bounded. All the assumptions are commonly seen in  asynchronous methods \cite{pedregosa2017breaking}.  
\end{remark}

\begin{figure*}[t]
\centering
\subfigure[KDD2010]{\includegraphics[width=0.28\textwidth]{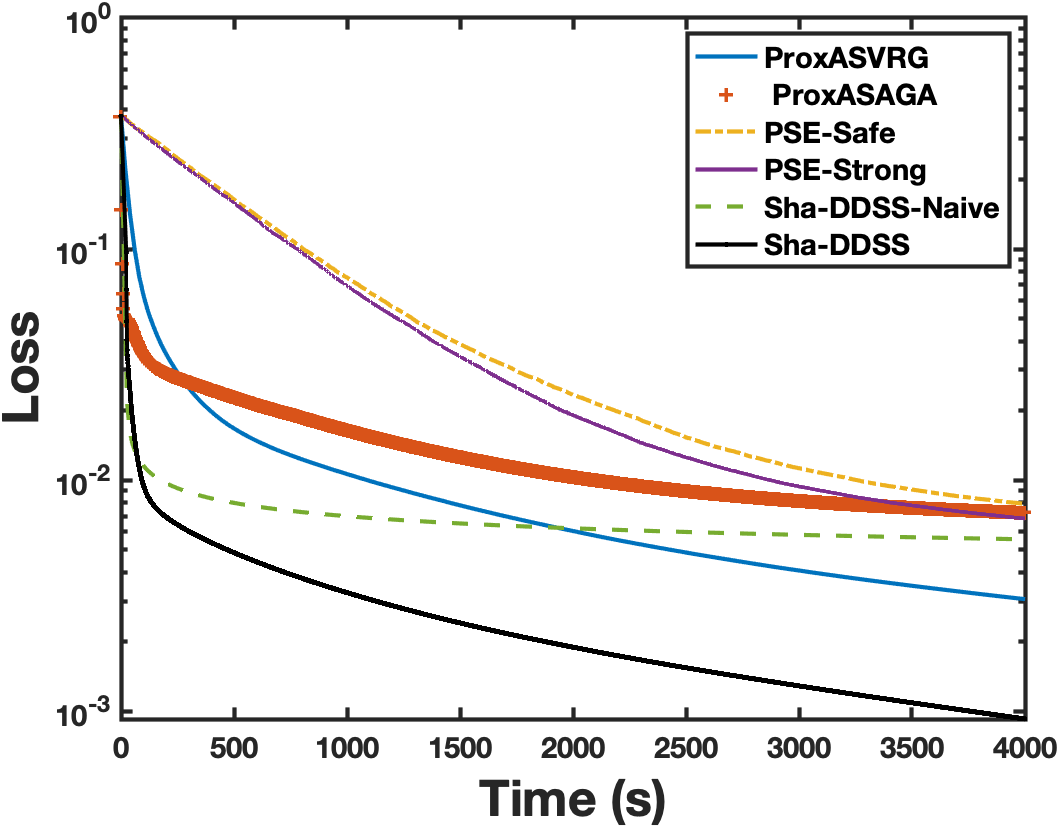}}
\hspace{9mm}
\subfigure[Avazu-app]{\includegraphics[width=0.28\textwidth]{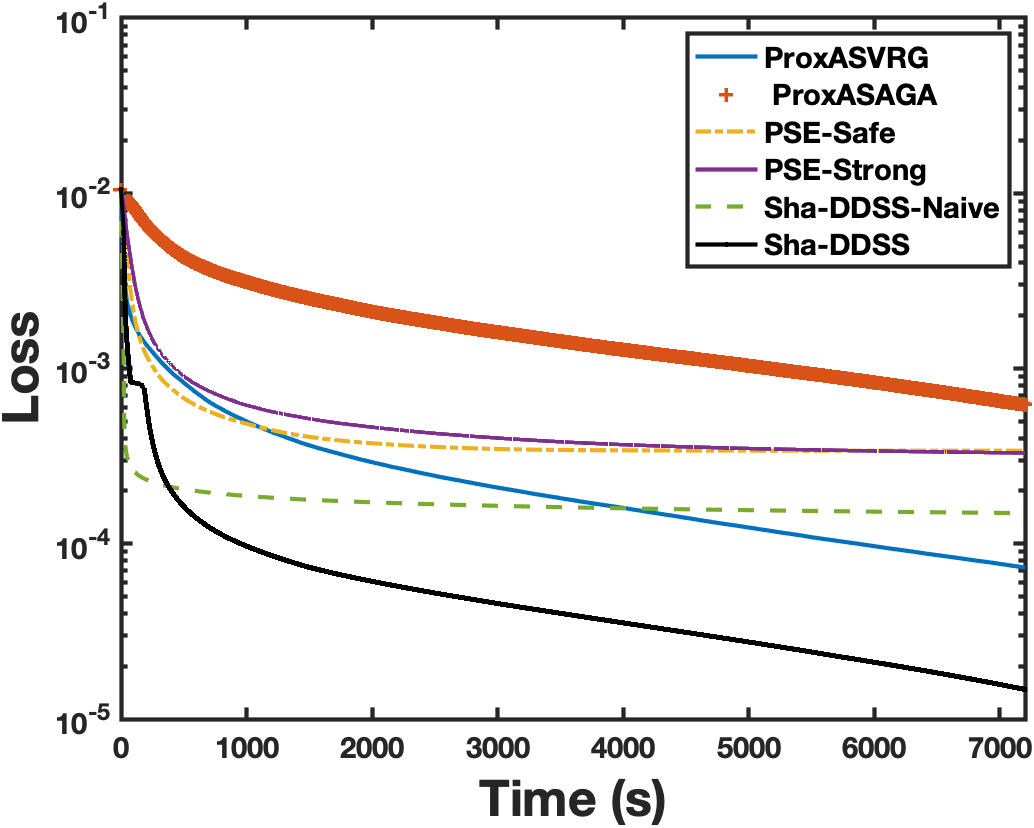}}
\hspace{9mm}
\subfigure[Avazu-site]{\includegraphics[width=0.28\textwidth]{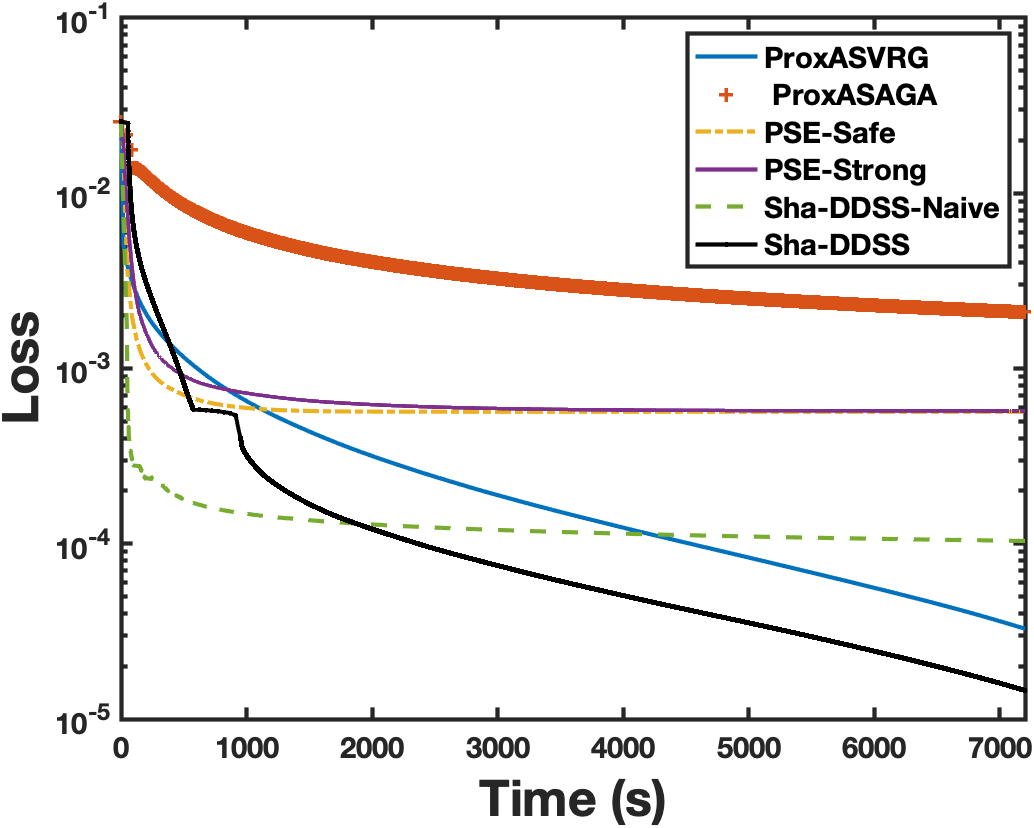}}
\caption{Convergence results  on shared-memory  architecture with $8$ threads.}
\label{fig1}
\end{figure*}

\begin{figure*}[t]
\centering
\subfigure[KDD2010]{\includegraphics[width=0.28\textwidth]{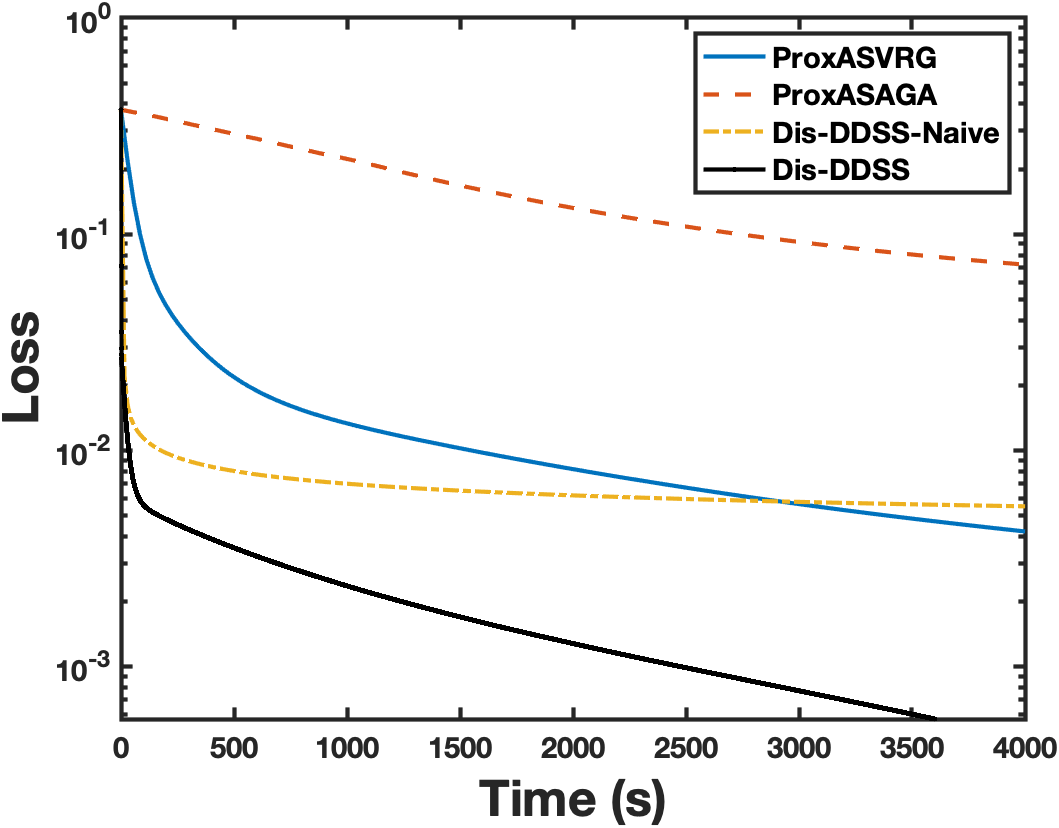}}
\hspace{9mm}
\subfigure[Avazu-app]{\includegraphics[width=0.28\textwidth]{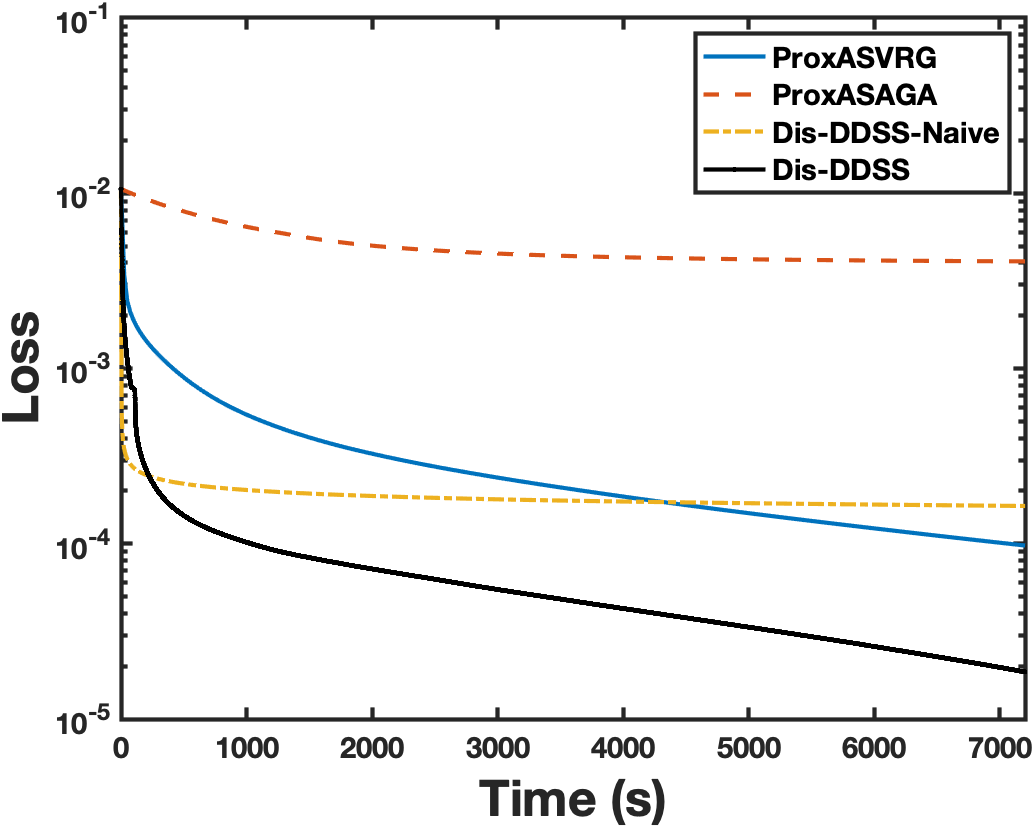}}
\hspace{9mm}
\subfigure[Avazu-site]{\includegraphics[width=0.28\textwidth]{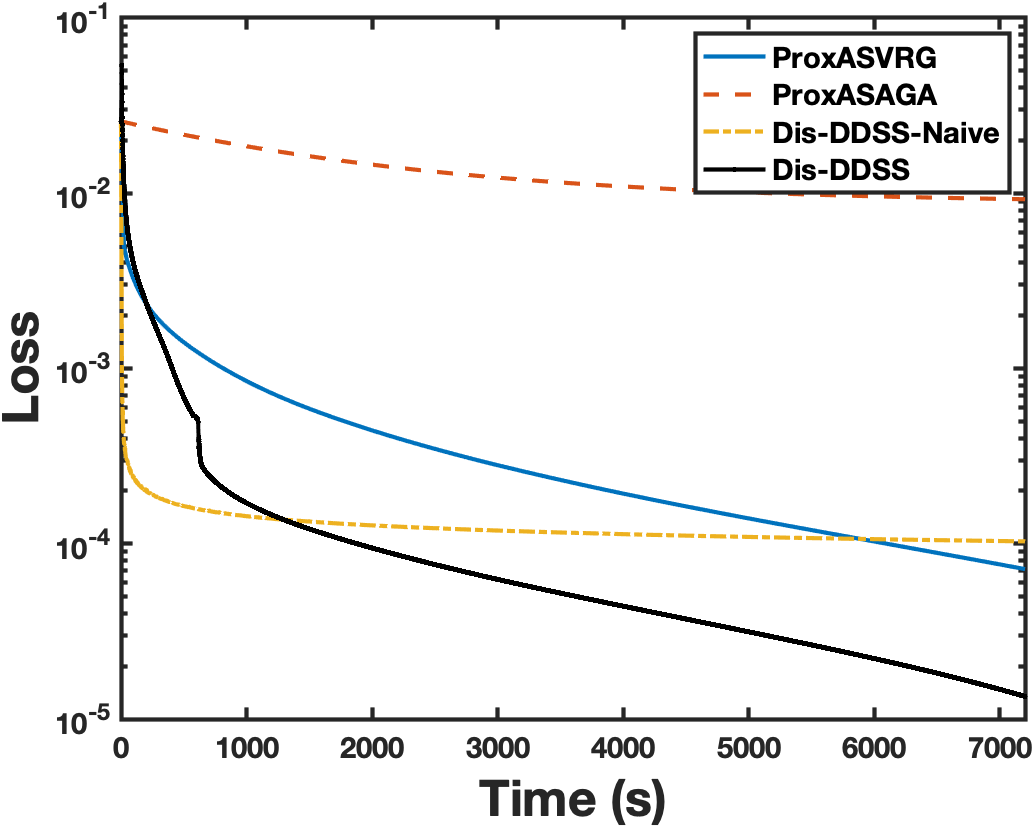}}
\caption{Convergence results  on distributed-memory  architecture with $8$ workers.}
\label{fig2}
\end{figure*}

\begin{figure*}[t]
\centering
\subfigure[KDD2010]{\includegraphics[width=0.29\textwidth]{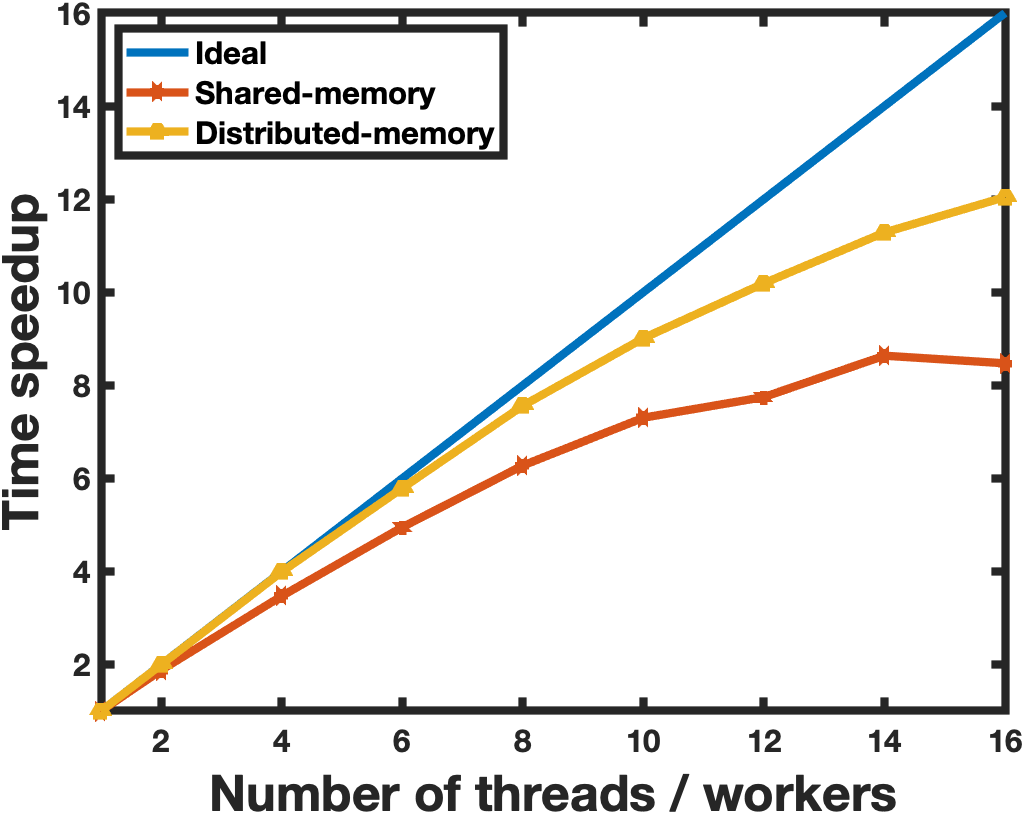}}
\hspace{9mm}
\subfigure[Avazu-app]{\includegraphics[width=0.29\textwidth]{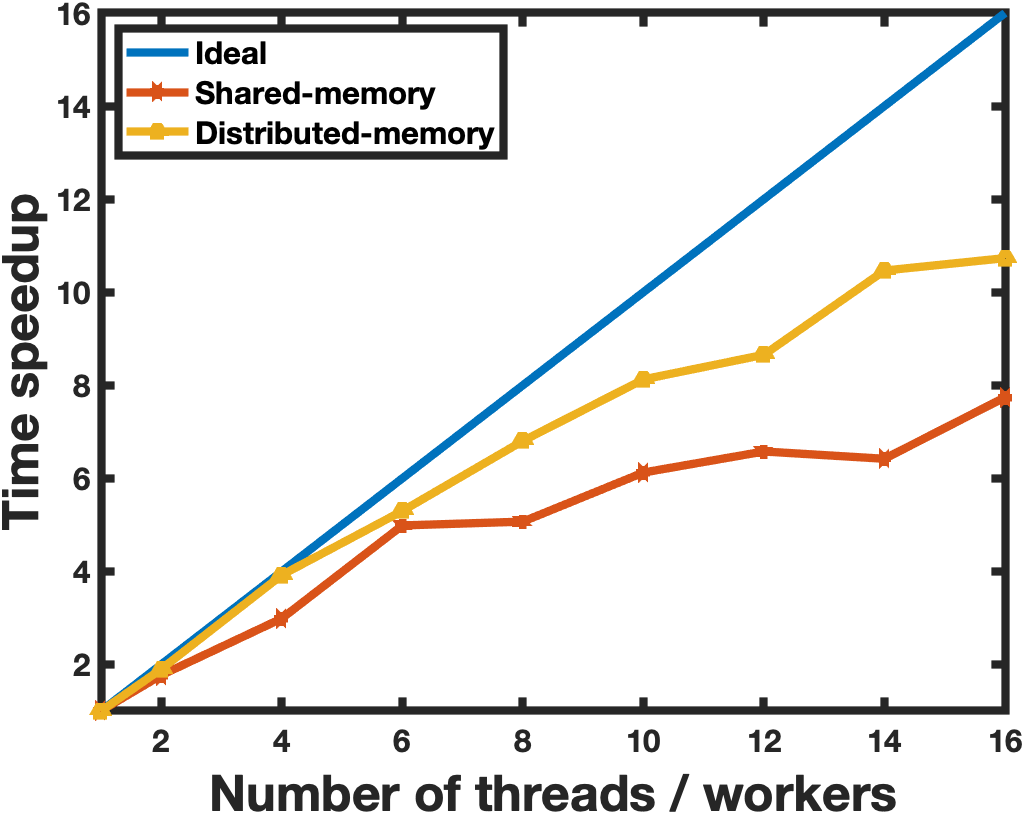}}
\hspace{9mm}
\subfigure[Avazu-site]{\includegraphics[width=0.29\textwidth]{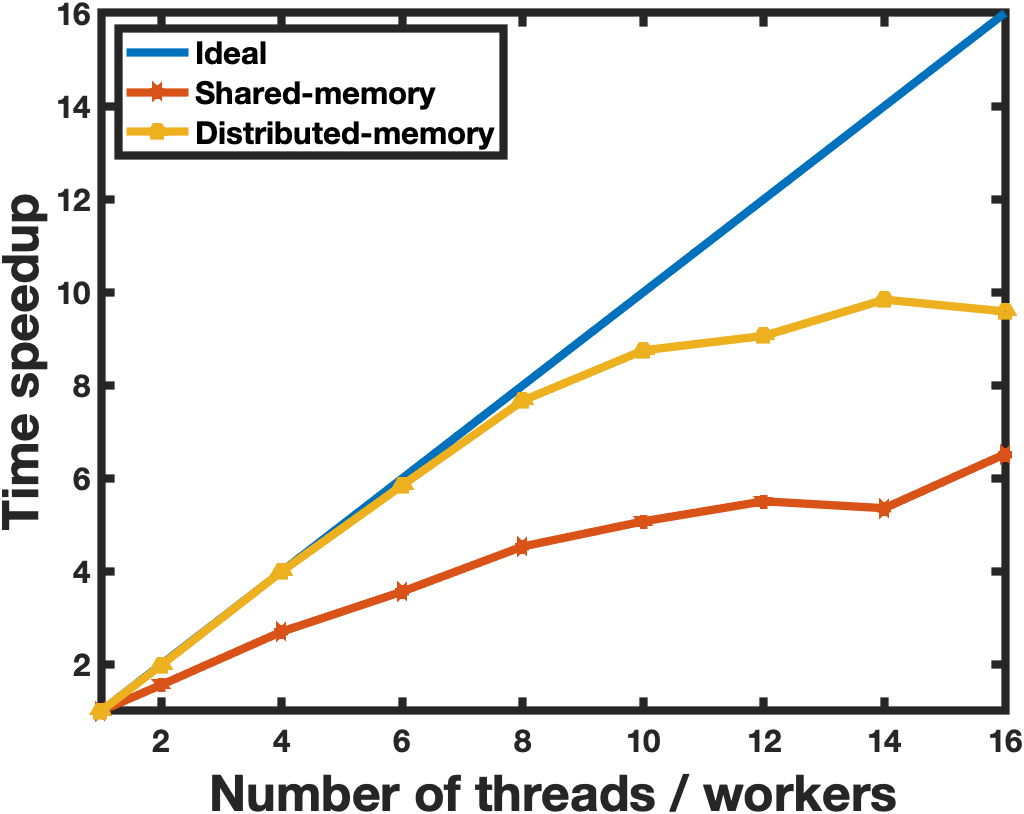}}
\caption{Convergence results  on shared-memory  architecture with $8$ threads.}
\label{fig3}
\end{figure*}

\begin{definition} [Block Sparsity]
We denote that the maximum frequency of occurrences $\Delta$ that a feature block belongs to the extended support, which can be formally defined as $\Delta = \max_{\G \in \B} |\{i: \Psi_i \ni \G \}| / n $. It is easy to verify that  $ 1/n \leq \Delta \leq 1$. 
\end{definition}

\begin{property} [Independence]
We use the “after read” labeling in \cite{leblond2017asaga}, which means we update the iterate counter after each thread fully reads the parameters. This means that $\hat{x}^{t}_{\B_{s+1}}$
is the $(t + 1)$-th fully completed read. Given the “after read”
global time counter, sample $i_r$ is independent of $\hat{x}^{t}_{\B_{s+1}}, \forall r\geq t$. 
\label{property:independence}
\end{property}

\begin{property} [Unbiased Gradient Estimation]
Gradient $v_t$ is an unbiased estimation of the gradient over set $\B_{s+1}$ at $\hat{x}^{t}_{\B_{s+1}}$, which is directly derived from Property \ref{property:independence}.
\end{property}

\begin{property}[Atomic Operation]
The update $x_{\B_{s+1}}^{t+1} = x_{\B_{s+1}}^{t} + \delta^s_{t}$ to shared-memory in Alg. \ref{algorithm:Sha-DDSS} is coordinate-wise atomic, which can address the overwriting problem caused by other threads. 
\end{property}

\subsection{Theoretical Results}

\begin{theorem} [Convergence]
Suppose $\tau \le \frac{1}{10\sqrt{\Delta}}$, let step size $\eta = \min \{ \frac{1}{24\kappa L}, \frac{\kappa}{2L}, \frac{\kappa}{10\tau L} \}$, inner loop size $K = \frac{4\log 3}{\eta\mu}$, we have
\begin{eqnarray}
    {\E} \left\|x_{\B_{S}}-x_{\B_{S}}^{*}\right\|^{2} \le (2/3)^{S} \left\| x_{{0}} - {x}^{*} \right\|^{2}.
\end{eqnarray}
\label{theorem:convergence}
\end{theorem}

\begin{remark}
Theorem \ref{theorem:convergence} shows that DDSS can achieve a linear convergence rate $\log(1/\epsilon)$. 
\end{remark}

\begin{remark}
For the case $\F(x)$ is nonstrongly convex, we can slightly modify $\Omega(x)$ by adding a small perturbation, e.g., $\mu_f\|x\|^2$ for smoothing where $\mu_f$ is a positive parameter. We can treat $\F(x) +\mu_f \|x\|^2$ as the data-fitting loss and then the loss is $\mu_f$-strongly convex. Denote $\kappa := L/\mu_f$, suppose $\tau \le \frac{1}{10\sqrt{\Delta}}$, let  $\eta = \min \{ \frac{1}{24\kappa L}, \frac{\kappa}{2L}, \frac{\kappa}{10\tau L} \}$,  $K = \frac{4\log 3}{\eta\mu_f}$, we have
${\E} \left\|x_{\B_{S}}-x_{\B_{S}}^{*}\right\|^{2} \le (2/3)^{S} \left\| x_{{0}} - {x}^{*} \right\|^{2}.$
Similarly, the overall computational cost is also very efficient. 
\end{remark}

\begin{theorem} [Screening Ability]
Equicorrelation set \cite{tibshirani2013lasso} is defined as 
$
\B^{*}:=\{j \in \{1,2,\ldots,q\}: \Omega_{j}^{D}(A_{j}^{\top} y^{*})=  n \lambda\}. $ Then, as DDSS converges, there exists an iteration number $S_{0} \in \mathbb{N}$, s.t. $\forall s \geq S_{0}$, any variable block $j \notin \B^{*}$ is eliminated by DDSS almost surely. 
\label{theorem:screening}
\end{theorem}

\begin{remark}
Suppose the size of set $\B_{s}$ is $p_s$ and $p^*$ is the size of the active features in $\B^*$, Theorem \ref{theorem:screening} shows we have $p_s$ is decreasing and $\lim_{s \rightarrow+\infty} p_s = p^*$.  
\end{remark}

\begin{remark}
To sum it up, by the elimination, the cost at each iteration is reduced from $O(p)$ to $O(p_s)$. Moreover, by the sparse update, only the nonzero coefficients of set $\B_{s}$ is updated and thus the cost at each iteration is further reduced from $O(p_s)$ to $O(p'_s)$ where $p'_s$ is the size of the nonzero coefficients. In the high-dimensional setting, we have $p^*\ll p$, $p_s \ll p$, and $p'_s \ll p$. Thus, with constantly decreasing $p_s$ and the sparse update, our DDSS can reduced the complexity from $O(p)$ to $O(r)$ where $r$ is the mean of $p'_s$ for $s = 1, 2, \ldots$, which can accelerate the training at a large extent in practice. 
\end{remark}

\section{EXPERIMENTS}

\subsection{Experimental Setup}
We compare our method with other competitive methods on three large-scale datasets. Although DDSS can work more broadly, we focus on Lasso for sparse regression, which is the most popular case for feature screening. Specifically, Lasso solves 
\begin{eqnarray}
\underset{x \in \Re^{p}}{\min} \frac{1}{n} \sum_{i=1}^{n} \frac{1}{2}(y_{i}-a_{i}^\top x)^{2}+\lambda\|x\|_{1}.
\label{eq:lasso}
\end{eqnarray}

On shared-memory architecture, we compare six asynchronous methods: 1)  PSE-Strong: parallel strong screening in \cite{li2016parallel};  2)  PSE-Safe: parallel static safe screening in \cite{li2016parallel}; 3) ProxASAGA \cite{pedregosa2017breaking}; 4) ProxASVRG \cite{meng2017asynchronous}; 5) Sha-DDSS-Naive; 6) Our Sha-DDSS. PSE-Safe and PSE-Strong are parallel static screening. ProxASAGA and ProxASVRG are popular asynchronous method with linear convergence. 

On distributed-memory architecture, we compare four asynchronous methods: 1) ProxASAGA \cite{leblond2018improved}; 2) ProxASVRG  \cite{meng2017asynchronous}; 3) Dis-DDSS-Naive; 4) Our Dis-DDSS. 

\begin{table}[h!]
\centering
\caption{Real-world datasets in the experiments.}
\setlength{\tabcolsep}{3.0mm}
\begin{tabular}{cccc}
\toprule
Dataset & Sample Size & Attributes & Sparsity   \\
\midrule
    KDD 2010 & 19,264,097 & 1,163,024  & $7 \times 10^{-6}$   \\
    Avazu-app   & 14,596,137       & 1,000,000  & $10^{-5}$\\
   Avazu-site    & 25,832,830 & 1,000,000  & $10^{-5}$ \\
\bottomrule
	\end{tabular}
\label{table:datasets}
\end{table}

We use three large-scale real-world benchmark datasets described in Table \ref{table:datasets}. All the datasets are from LIBSVM \cite{chang2011libsvm}, which can be found at \url{https://www.csie.ntu.edu.tw/~cjlin/libsvmtools/datasets/}. 

We implement all the compared methods in C++. We employ OpenMP and OpenMPI as the parallel framework for shared-memory and distributed-memory architecture respectively.  We run all the methods on 2.10 GHz Intel(R) Xeon(R) CPU machines. For the implementation, the inner loop size, ranging from $2 \times 10^{3}$ to $2 \times 10^{6}$, and the step size, ranging from $10^{-11}$ to $10^{-13}$,  of each method are chosen to obtain the best performance.  Parameter $\lambda$ is set as $4*10^{-6} \lambda_{max}$, $2*10^{-3} \lambda_{max}$, and $1*10^{-3} \lambda_{max}$ for KDD 2010, Avazu-app, and Avazu-site dataset respectively where $\lambda_{\max}$ is a parameter that, for all $\lambda \geq \lambda_{\max}$,  $x^*$  must be $0$.

\subsection{Experimental Results and Discussions}

\paragraph{Convergence Results}
Figure \ref{fig1} (a)-(c) provides the convergence results of different methods on shared-memory architecture with $8$ threads  on three datasets respectively. Our Sha-DDSS-Naive method converges very fast at the initial stage because of the screening ability and slows down later due to its sublinear convergence rate.  The results confirm that our Sha-DDSS method always converge much faster than other methods on shared-memory architecture. Figure \ref{fig2} (a)-(c) provides the convergence results of different methods on distributed-memory architecture with $8$ workers  on three datasets respectively. The results also confirm that our Dis-DDSS always converge much faster than other methods on distributed-memory architecture. 

This is because our method on both shared-memory and distributed-memory architecture can eliminate the features by exploiting the sparsity of the model, perform efficient sparse update by  exploiting the sparsity of the dataset, achieve the linear convergence rate by reducing the gradient variance. Our Dis-DDSS also performs the decouple proximal update to reduce the workload of the server and reduces the communication costs.

\paragraph{Linear Speedup Property}
We evaluate Sha-DDSS with different number of threads on shared-memory architecture and Dis-DDSS with different number of workers on distributed-memory architecture. Figure \ref{fig3}(a)-(c) presents the results of the speedup of  Sha-DDSS and Dis-DDSS on three datasets respectively. The results show that our method can successfully achieve a nearly linear speedup when we increase the number of threads or workers, although the performance decreases when the number of processors or works increases. This is because there are overheads for creating threads and distributing work for OpenMP and communication costs for OpenMPI, which the theoretical analysis does not take into account.

\section{CONCLUSION}
In this paper, we propose the first distributed dynamic safe screening method for sparse models and apply it on shared-memory and distributed-memory architecture respectively. Theoretically, we prove that our proposed method can achieve a linear convergence rate with lower overall complexity.  Moreover, we prove that our method can eliminate almost all the inactive variables in a finite number of iterations almost surely.
Finally, extensive experimental results on benchmark datasets confirm the significant acceleration
and linear speedup property of our method.

\iffalse
If a line is just slightly longer than the column width, you may use the {\tt resizebox} environment on that equation. The result looks better and doesn't interfere with the paragraph's line spacing: %
\begin{equation}
\resizebox{.91\linewidth}{!}{$
    \displaystyle
    x = \prod_{i=1}^n \sum_{j=1}^n j_i + \prod_{i=1}^n \sum_{j=1}^n i_j + \prod_{i=1}^n \sum_{j=1}^n j_i + \prod_{i=1}^n \sum_{j=1}^n i_j + \prod_{i=1}^n \sum_{j=1}^n j_i
$}
\end{equation}%
\fi

\newpage
\newpage

\bibliographystyle{named}
\bibliography{ijcai22}

\begin{thebibliography}{}

\bibitem[\protect\citeauthoryear{Agarwal and
  Duchi}{2012}]{agarwal2012distributed}
Alekh Agarwal and John~C Duchi.
\newblock Distributed delayed stochastic optimization.
\newblock In {\em IEEE CDC}, 2012.

\bibitem[\protect\citeauthoryear{Bao \bgroup \em et al.\egroup
  }{2019}]{bao2019efficient}
Runxue Bao, Bin Gu, and Heng Huang.
\newblock Efficient approximate solution path algorithm for ordered weighted
  l\_1-norm with accuracy guarantee.
\newblock In {\em IEEE ICDM}, 2019.

\bibitem[\protect\citeauthoryear{Bao \bgroup \em et al.\egroup
  }{2020}]{bao2020fast}
Runxue Bao, Bin Gu, and Heng Huang.
\newblock Fast oscar and owl regression via safe screening rules.
\newblock In {\em ICML}, 2020.

\bibitem[\protect\citeauthoryear{Bian \bgroup \em et al.\egroup
  }{2021}]{bian2021optimization}
Wanyu Bian, Yunmei Chen, Xiaojing Ye, et~al.
\newblock An optimization-based meta-learning model for mri reconstruction with
  diverse dataset.
\newblock {\em J. Imaging}, 2021.

\bibitem[\protect\citeauthoryear{Chang and Lin}{2011}]{chang2011libsvm}
Chih-Chung Chang and Chih-Jen Lin.
\newblock Libsvm: A library for support vector machines.
\newblock {\em ACM TIST}, 2011.

\bibitem[\protect\citeauthoryear{Chen \bgroup \em et al.\egroup
  }{2021}]{chen2021learnable}
Yunmei Chen, Hongcheng Liu, Xiaojing Ye, et~al.
\newblock Learnable descent algorithm for nonsmooth nonconvex image
  reconstruction.
\newblock {\em SIAM J. Imaging Sci.}, 2021.

\bibitem[\protect\citeauthoryear{Dean \bgroup \em et al.\egroup
  }{2012}]{dean2012large}
Jeffrey Dean, Greg Corrado, Rajat Monga, et~al.
\newblock Large scale distributed deep networks.
\newblock In {\em NeurIPS}, 2012.

\bibitem[\protect\citeauthoryear{Fercoq \bgroup \em et al.\egroup
  }{2015}]{fercoq2015mind}
Olivier Fercoq, Alexandre Gramfort, and Joseph Salmon.
\newblock Mind the duality gap: safer rules for the lasso.
\newblock In {\em ICML}, 2015.

\bibitem[\protect\citeauthoryear{Gu \bgroup \em et al.\egroup
  }{2018}]{pmlr-v84-gu18a}
Bin Gu, Zhouyuan Huo, and Heng Huang.
\newblock Asynchronous doubly stochastic group regularized learning.
\newblock In {\em AISTATS}, 2018.

\bibitem[\protect\citeauthoryear{Langford \bgroup \em et al.\egroup
  }{2009}]{langford2009slow}
John Langford, Alexander~J Smola, and Martin Zinkevich.
\newblock Slow learners are fast.
\newblock In {\em NeurIPS}, 2009.

\bibitem[\protect\citeauthoryear{Laurent El~Ghaoui}{2012}]{Laurent2012safe}
Tarek~Rabbani Laurent El~Ghaoui, Vivian~Viallon.
\newblock Safe feature elimination in sparse supervised learning.
\newblock {\em Pacific J. Optim.}, 2012.

\bibitem[\protect\citeauthoryear{Leblond \bgroup \em et al.\egroup
  }{2017}]{leblond2017asaga}
R{\'e}mi Leblond, Fabian Pedregosa, and Simon Lacoste-Julien.
\newblock Asaga: asynchronous parallel saga.
\newblock In {\em AISTATS}, 2017.

\bibitem[\protect\citeauthoryear{Leblond \bgroup \em et al.\egroup
  }{2018}]{leblond2018improved}
R{\'e}mi Leblond, Fabian Pedregosa, and Simon Lacoste-Julien.
\newblock Improved asynchronous parallel optimization analysis for stochastic
  incremental methods.
\newblock {\em JMLR}, 2018.

\bibitem[\protect\citeauthoryear{Li \bgroup \em et al.\egroup
  }{2016a}]{li2016parallel}
Qingyang Li, Shuang Qiu, Shuiwang Ji, et~al.
\newblock Parallel lasso screening for big data optimization.
\newblock In {\em ACM SIGKDD}, 2016.

\bibitem[\protect\citeauthoryear{Li \bgroup \em et al.\egroup
  }{2016b}]{li2016make}
Yitan Li, Linli Xu, Xiaowei Zhong, and Qing Ling.
\newblock Make workers work harder: decoupled asynchronous proximal stochastic
  gradient descent.
\newblock {\em arXiv preprint arXiv:1605.06619}, 2016.

\bibitem[\protect\citeauthoryear{Li \bgroup \em et al.\egroup
  }{2021}]{li2021fully}
Junyi Li, Bin Gu, and Heng Huang.
\newblock A fully single loop algorithm for bilevel optimization without
  hessian inverse.
\newblock {\em arXiv preprint arXiv:2112.04660}, 2021.

\bibitem[\protect\citeauthoryear{Lian \bgroup \em et al.\egroup
  }{2015}]{lian2015asynchronous}
Xiangru Lian, Yijun Huang, Yuncheng Li, and Ji~Liu.
\newblock Asynchronous parallel stochastic gradient for nonconvex optimization.
\newblock In {\em NeurIPS}, 2015.

\bibitem[\protect\citeauthoryear{Lustig \bgroup \em et al.\egroup
  }{2008}]{lustig2008compressed}
Michael Lustig, David~L Donoho, Juan~M Santos, and John~M Pauly.
\newblock Compressed sensing mri.
\newblock {\em IEEE Signal Process. Mag.}, 2008.

\bibitem[\protect\citeauthoryear{Mania \bgroup \em et al.\egroup
  }{2017}]{mania2017perturbed}
Horia Mania, Xinghao Pan, Dimitris Papailiopoulos, et~al.
\newblock Perturbed iterate analysis for asynchronous stochastic optimization.
\newblock {\em SIOPT}, 2017.

\bibitem[\protect\citeauthoryear{Meng \bgroup \em et al.\egroup
  }{2017}]{meng2017asynchronous}
Qi~Meng, Wei Chen, Jingcheng Yu, et~al.
\newblock Asynchronous stochastic proximal optimization algorithms with
  variance reduction.
\newblock In {\em AAAI}, 2017.

\bibitem[\protect\citeauthoryear{Ndiaye \bgroup \em et al.\egroup
  }{2015}]{ndiaye2015gap}
Eugene Ndiaye, Olivier Fercoq, Alexandre Gramfort, and Joseph Salmon.
\newblock Gap safe screening rules for sparse multi-task and multi-class
  models.
\newblock {\em NeurIPS}, 2015.

\bibitem[\protect\citeauthoryear{Ndiaye \bgroup \em et al.\egroup
  }{2016}]{ndiaye2016gap}
Eugene Ndiaye, Olivier Fercoq, Alexandre Gramfort, and Joseph Salmon.
\newblock Gap safe screening rules for sparse-group lasso.
\newblock In {\em NeurIPS}, 2016.

\bibitem[\protect\citeauthoryear{Ndiaye \bgroup \em et al.\egroup
  }{2017}]{ndiaye2017gap}
Eugene Ndiaye, Olivier Fercoq, Alexandre Gramfort, et~al.
\newblock Gap safe screening rules for sparsity enforcing penalties.
\newblock {\em JMLR}, 2017.

\bibitem[\protect\citeauthoryear{Ng}{2004}]{ng2004feature}
Andrew~Y Ng.
\newblock Feature selection, l1 vs. l2 regularization, and rotational
  invariance.
\newblock In {\em ICML}, 2004.

\bibitem[\protect\citeauthoryear{Pedregosa \bgroup \em et al.\egroup
  }{2017}]{pedregosa2017breaking}
Fabian Pedregosa, R{\'e}mi Leblond, and Simon Lacoste-Julien.
\newblock Breaking the nonsmooth barrier: A scalable parallel method for
  composite optimization.
\newblock In {\em NeurIPS}, 2017.

\bibitem[\protect\citeauthoryear{Rakotomamonjy \bgroup \em et al.\egroup
  }{2019}]{rakotomamonjy2019screening}
Alain Rakotomamonjy, Gilles Gasso, and Joseph Salmon.
\newblock Screening rules for lasso with non-convex sparse regularizers.
\newblock In {\em ICML}, 2019.

\bibitem[\protect\citeauthoryear{Recht \bgroup \em et al.\egroup
  }{2011}]{recht2011hogwild}
Benjamin Recht, Christopher R{\'e}, Stephen~J Wright, et~al.
\newblock Hogwild: A lock-free approach to parallelizing stochastic gradient
  descent.
\newblock In {\em NeurIPS}, 2011.

\bibitem[\protect\citeauthoryear{Reddi \bgroup \em et al.\egroup
  }{2015}]{reddi2015variance}
Sashank~J Reddi, Ahmed Hefny, Suvrit Sra, et~al.
\newblock On variance reduction in stochastic gradient descent and its
  asynchronous variants.
\newblock In {\em NeurIPS}, 2015.

\bibitem[\protect\citeauthoryear{Shevade and Keerthi}{2003}]{shevade2003simple}
Shirish~Krishnaj Shevade and S~Sathiya Keerthi.
\newblock A simple and efficient algorithm for gene selection using sparse
  logistic regression.
\newblock {\em Bioinformatics}, 2003.

\bibitem[\protect\citeauthoryear{Shibagaki \bgroup \em et al.\egroup
  }{2016}]{shibagaki2016simultaneous}
Atsushi Shibagaki, Masayuki Karasuyama, Kohei Hatano, and Ichiro Takeuchi.
\newblock Simultaneous safe screening of features and samples in doubly sparse
  modeling.
\newblock In {\em ICML}, 2016.

\bibitem[\protect\citeauthoryear{Tibshirani \bgroup \em et al.\egroup
  }{2012}]{tibshirani2012strong}
Robert Tibshirani, Jacob Bien, Jerome Friedman, et~al.
\newblock Strong rules for discarding predictors in lasso-type problems.
\newblock {\em JRSS}, 2012.

\bibitem[\protect\citeauthoryear{Tibshirani}{1996}]{tibshirani1996regression}
Robert Tibshirani.
\newblock Regression shrinkage and selection via the lasso.
\newblock {\em JRSS}, 1996.

\bibitem[\protect\citeauthoryear{Tibshirani}{2013}]{tibshirani2013lasso}
Ryan~J Tibshirani.
\newblock The lasso problem and uniqueness.
\newblock {\em Electron. J. Stat.}, 2013.

\bibitem[\protect\citeauthoryear{Wang \bgroup \em et al.\egroup
  }{2013}]{wang2013lasso}
Jie Wang, Jiayu Zhou, Peter Wonka, and Jieping Ye.
\newblock Lasso screening rules via dual polytope projection.
\newblock In {\em NeurIPS}, 2013.

\bibitem[\protect\citeauthoryear{Wright \bgroup \em et al.\egroup
  }{2009}]{wright2009sparse}
Stephen~J Wright, Robert~D Nowak, and M{\'a}rio~AT Figueiredo.
\newblock Sparse reconstruction by separable approximation.
\newblock {\em IEEE Trans. Signal Process.}, 2009.

\bibitem[\protect\citeauthoryear{Wright \bgroup \em et al.\egroup
  }{2010}]{wright2010sparse}
John Wright, Yi~Ma, Julien Mairal, et~al.
\newblock Sparse representation for computer vision and pattern recognition.
\newblock {\em Proc. IEEE}, 2010.

\bibitem[\protect\citeauthoryear{Xiao and Zhang}{2014}]{xiao2014proximal}
Lin Xiao and Tong Zhang.
\newblock A proximal stochastic gradient method with progressive variance
  reduction.
\newblock {\em SIOPT}, 2014.

\bibitem[\protect\citeauthoryear{Yuan and Lin}{2006}]{yuan2006model}
Ming Yuan and Yi~Lin.
\newblock Model selection and estimation in regression with grouped variables.
\newblock {\em JRSS}, 2006.

\bibitem[\protect\citeauthoryear{Zhang and Kwok}{2014}]{zhang2014asynchronous}
Ruiliang Zhang and James Kwok.
\newblock Asynchronous distributed admm for consensus optimization.
\newblock In {\em ICML}, 2014.

\bibitem[\protect\citeauthoryear{Zhang \bgroup \em et al.\egroup
  }{2016}]{zhang2016asynchronous}
Ruiliang Zhang, Shuai Zheng, and James~T Kwok.
\newblock Asynchronous distributed semi-stochastic gradient optimization.
\newblock In {\em AAAI}, 2016.

\bibitem[\protect\citeauthoryear{Zhang \bgroup \em et al.\egroup
  }{2021}]{zhang2021ddn2}
Bai Zhang, Yi~Fu, Yingzhou Lu, et~al.
\newblock Ddn2.0: R and python packages for differential dependency network
  analysis of biological systems.
\newblock {\em bioRxiv}, 2021.

\bibitem[\protect\citeauthoryear{Zhao and Li}{2016}]{zhao2016fast}
Shen-Yi Zhao and Wu-Jun Li.
\newblock Fast asynchronous parallel stochastic gradient descent: A lock-free
  approach with convergence guarantee.
\newblock In {\em AAAI}, 2016.

\bibitem[\protect\citeauthoryear{Zhou \bgroup \em et al.\egroup
  }{2018}]{zhou2018simple}
Kaiwen Zhou, Fanhua Shang, and James Cheng.
\newblock A simple stochastic variance reduced algorithm with fast convergence
  rates.
\newblock In {\em ICML}, 2018.

\bibitem[\protect\citeauthoryear{Zhou}{2018}]{zhou2018fenchel}
Xingyu Zhou.
\newblock On the fenchel duality between strong convexity and lipschitz
  continuous gradient.
\newblock {\em arXiv preprint arXiv:1803.06573}, 2018.

\end{thebibliography}

\onecolumn

\appendix

\newpage

\section{\Large Appendix}

In the appendix, we provide some basic lemmas and the proof for all the theorems.

\subsection{Basic Lemmas}

\begin{lemma}
\label{lm1}
Suppose $\F$ is $\mu$-strongly convex, we have:
\begin{eqnarray} 
    \langle\nabla \F(y)-\nabla \F(x), y-x\rangle \geq \frac{\mu}{2}\|y-x\|^{2}+B_{\F}(x, y)
\end{eqnarray} 
where $B_{\F}(x, y)$  is the Bregman divergence  defined as $B_{\F}(x, y):=\F(x)-\F(y)-\langle\nabla \F(y), x-y\rangle.$
\end{lemma}

\begin{proof} By strong convexity, for any $x, y$,  we have:
\begin{eqnarray}
&&    \F(y) \geq \F(x) + \langle \nabla \F(x),  y-x  \rangle + \frac{\mu}{2}\|y-x\|^2 \nonumber \\
&\iff& \langle \nabla \F(x), x-y \rangle \geq \frac{\mu}{2}\|y-x\|^2 + \F(x)-\F(y)   \nonumber \\
&\iff&  \langle \nabla \F(x) - \nabla \F(y), x-y  \rangle \geq \frac{\mu}{2}\|y-x\|^2 + \F(x) - \F(y) - \langle \nabla \F(y), x-y \rangle \nonumber \\
&\iff&  \langle \nabla \F(y) - \nabla \F(x), x-y  \rangle \geq \frac{\mu}{2}\|x-y\|^2 + B_{\F}(x, y).
\end{eqnarray}
This finishes the proof.
\end{proof}

\begin{lemma}
\label{lm2}
Suppose $\F_i$ is $L$-smooth and convex, we have:
\begin{eqnarray} 
    \frac{1}{n} \sum_{i=1}^{n}\left\|\nabla \F_i(x)-\nabla \F_i(y)\right\|^{2} \leq 2 L B_{\F}(x, y).
\end{eqnarray} 

\end{lemma}
\begin{proof}
Based on Eq. $5$ of Lemma 4 in \cite{zhou2018fenchel}, we have
\begin{eqnarray} 
\|\nabla \F_i(x) - \nabla \F_i(y)\|^2 \leq 2 L \big(\F_i(x) - \F_i(y) - \langle \nabla \F_i(y), x - y\rangle\big)\,.
\end{eqnarray} 
Averaging with $i$, we have 
\begin{eqnarray} 
 \frac{1}{n} \sum_{i=1}^{n} \|\nabla \F_i(x) - \nabla \F_i(y)\|^2 \leq 2 L \big(\F(x) - \F(y) - \langle \nabla \F(y), x - y\rangle\big)
 \end{eqnarray} 
which is equivalent to
\begin{eqnarray} 
\langle \nabla \F(x) - \nabla \F(y), x - y \rangle \geq \frac{\mu}{2}\|x - y\|^2 + B_{\F}(x, y) .
\end{eqnarray} 
This finishes the proof.
\end{proof}

\begin{lemma}
\label{lm3}
Let $x^*$ be the optimal solution of Problem (\ref{eq:general}), $v_t^s$ is the sparse variance reduced gradient with sample $i$ defined in (\ref{eq:vr_gradient}), $g$ is the the sparse gradient mapping for $v_t^s$ and computed as $ g = \frac{1}{\eta}(x - \operatorname{prox}_{\eta \lambda \phi_{i}}\left(x - \eta v_t^s \right) )$. Then, for any $\beta > 0$ and $x \in \Re^{p}$, we have:
\begin{eqnarray} 
\left\langle{g}, {x}-{x}^{*}\right\rangle \geq-\frac{\eta}{2}(\beta-2)\|{g}\|^{2}-\frac{\eta}{2 \beta}\left\| v_t^s -{D}_{i} \nabla \F\left({x}^{*}\right)\right\|^{2}+\left\langle v_t^s -{D}_{i} \nabla \F\left({x}^{*}\right), {x}-{x}^{*}\right\rangle.
\end{eqnarray} 
\end{lemma}

\begin{proof}
Let $x^+ = \operatorname{prox}_{\eta \lambda \phi_i}(x - \eta v_t^s)$ and $x^* = \operatorname{prox}_{\eta \lambda \phi_i}(x^* - \eta D  \nabla \F(x^*))$, denote $\langle \cdot, \cdot \rangle_{(i)}$ as the inner product restricted to the blocks in $\Psi_i$ and   $\|\cdot\|_{(i)}$ as the norm restricted to the blocks in $\Psi_i$,  we have:
\begin{eqnarray} \label{eq:nonexpansive}
    \|x^+ - x^*\|_{(i)}^2 - \langle x^+ - x^*, x - \eta v_t^s - x^* + \eta D \nabla \F(x^*) \rangle_{(i)} \leq 0,
\end{eqnarray} 
which comes from the firm non-expansiveness of the proximal operator.

Then we have:
\begin{eqnarray} 
 \langle \eta g, x - x^*\rangle 
    &=& \langle x - x^+, x - x^* \rangle_{(i)} \nonumber\\
    &=& \|x - x^*\|_{(i)}^2 -  \langle x^+ - x^*, x - x^* \rangle_{(i)} \nonumber\\
    &\geq& \|x - x^*\|_{(i)}^2 - \langle x^+ - x^*, 2 x - \eta v_t^s- 2 x^* + \eta D \nabla \F(x^*) \rangle_{(i)} + \|x^+ - x^*\|^2_{(i)} 
 \nonumber\\
    &=& \|x - x^+\|_{(i)}^2 + \langle x^+ - x^*, \eta v_t^s -  \eta D \nabla \F(x^*) \rangle_{(i)} \nonumber\\
    &=& \|x - x^+\|_{(i)}^2 + \langle x - x^*, \eta v_t^s -  \eta D \nabla \F(x^*) \rangle_{(i)}- \langle x - x^+, \eta v_t^s-  \eta D \nabla \F(x^*) \rangle_{(i)}  \\
    &\geq& \Big(1 - \frac{\beta}{2}\Big)\|x - x^+\|_{(i)}^2 -  \frac{\eta^2}{2\beta}\| v_t^s -  D \nabla \F(x^*) \|_{(i)}^2 + \eta\langle  v_t^s -  D \nabla \F(x^*) , x - x^*\rangle_{(i)} \nonumber\\
    & =& \Big(1 - \frac{\beta}{2}\Big)\|x - x^+\|_{(i)}^2 -  \frac{\eta^2}{2\beta}\| v_t^s -  D_i \nabla \F(x^*) \|^2 + \eta\langle  v_t^s -  D_i \nabla \F(x^*) , x - x^*\rangle \nonumber\\
    &=& \Big(1 - \frac{\beta}{2}\Big)\|\eta  g\|^2 -  \frac{\eta^2}{2\beta}\| v_t^s -  D_i \nabla \F(x^*) \|^2 + \eta\langle  v_t^s -  D_i \nabla \F(x^*) , x - x^*\rangle \,,
\end{eqnarray} 
where the first inequality is obtained by adding Eq.~\eqref{eq:nonexpansive}, the second inequality is obtained by Young's inequality ${2 \langle a, b \rangle \leq \frac{\|a\|^2}{\beta} + \beta \|b\|^2}$ for arbitrary $\beta > 0$. Finally, the result finishes the proof.
\end{proof}

\begin{lemma}\cite[Proposition 1]{leblond2017asaga}\label{lm4}
For any $u \neq t$, we have
\begin{eqnarray} \label{sparseproduct}
\E |\langle g_{u}, g_t \rangle | &\leq \frac{\sqrt{\Delta}}{2}(\E\|g_{u}\|^2 + \E\|g_{t}\|^2)  \, .
\end{eqnarray} 
\end{lemma}

\begin{lemma} \label{lm5}
We have following estimations:
\begin{eqnarray} 
\E\left\|\hat{{x}}_{t}^s - {x}_{t}^s \right\|^{2} \leq \eta^{2}(1+\sqrt{\Delta} \tau) \sum_{u=(t-\tau)_{+}}^{t-1} \E\left\|{g}_{u}^s\right\|^{2}
\end{eqnarray} 
\begin{eqnarray} 
\E\left\langle {g}_{t}^s, \hat{{x}}_{t}^s - {x}_{t}^s \right\rangle \leq \frac{\eta \sqrt{\Delta}}{2} \sum_{u=(t-\tau)_{+}}^{t-1} \E\left\| {g}_{u}^s \right\|^{2}+\frac{\eta \sqrt{\Delta} \tau}{2} \E\left\| {g}_{t}^s \right\|^{2}
\end{eqnarray} 
\end{lemma}

\begin{proof}
By Assumption \ref{assumption:overlapping}, we have the following updates:
\begin{eqnarray} \label{eq:async}
 \hat x_t - x_t = \eta \sum_{u=(t - \tau)_+}^{t-1}G_{u}^t g(\hat x_{u}, \hat \alpha^u, i_{u}),
 \end{eqnarray} 
Thus, we have:
\begin{eqnarray} 
\E\|\hat x_t - x_t\|^2
& \leq&  \eta^2 \sum_{u, v=(t -\tau)_+}^{t-1} \E |\langle G_u^t g_{u}, G_v^t g_{v}\rangle | 
\nonumber \\
& \leq &  \eta^2 \sum_{u=(t -\tau)_+}^{t-1}\E \|g_{u}\|^2 
	+ \eta^2 \sum_{\substack{u, v=(t-\tau)_+ \\u\neq v}}^{t-1} \E  |\langle G_u^t g_{u}, G_v^t g_{v}\rangle | 
\nonumber \\
& \leq&   \eta^2 \sum_{u=(t -\tau)_+}^{t-1}\E \|g_{u}\|^2 
	+ \eta^2 \sum_{\substack{u, v=(t-\tau)_+ \\u\neq v}}^{t-1} \E |\langle g_{u}, g_{v}\rangle |
\nonumber \\
&\leq& \eta^2 \sum_{u=(t-\tau)_+}^{t-1}\E\|g_{u}\|^2 
	+ \eta^2 \sqrt{\Delta}(\tau-1)_+ \sum_{u=(t-\tau)_+}^{t-1}\E\|g_{u}\|^2 
\nonumber \\
&=& \eta^2 \big(1+\sqrt{\Delta}(\tau-1)_+ \big)\sum_{u=(t-\tau)_+}^{t-1}\E\|g_{u}\|^2 
\nonumber \\
&\leq& \eta^2 \big(1+\sqrt{\Delta}\tau \big)\sum_{u=(t-\tau)_+}^{t-1}\E\|g_{u}\|^2 .
\end{eqnarray} 
where the fourth inequality is obtained by Lemma \ref{lm4}. 

Taking the expectation of $\langle \hat x_t -x_t,  g_t\rangle$, we have:
\begin{eqnarray} 
\frac{1}{\eta} \E\langle \hat x_t - x_t, g_t \rangle 
&=& \sum_{u=(t - \tau)_+}^{t-1} \E \langle G_u^t g_u, g_t \rangle \nonumber \\
&\leq& \sum_{u=(t - \tau)_+}^{t-1} \E | \langle g_u, g_t \rangle | \nonumber \\
&\leq& \sum_{u=(t - \tau)_+}^{t-1} \frac{\sqrt{\Delta}}{2}(\E\|g_{u}\|^2 + \E\|g_{t}\|^2) 
\nonumber \\
&\leq& \frac{\sqrt{\Delta}}{2} \sum_{u=(t - \tau)_+}^{t-1}\E\|g_{u}\|^2 + \frac{\sqrt{\Delta}\tau}{2}\E\|g_{t}\|^2.
\end{eqnarray} 
This finishes the proof.
\end{proof}

\begin{lemma} \label{lm6}
For any $x \in \Re^{p}$ and gradient estimator $v_t^s$ computed by sample $i$, we have
\begin{eqnarray} 
 \E \left\| v_t^s - D_{i} \nabla \F(x^*) \right\|^2 \le 4L \E B_\F(\hat{x}_t^s, x^*) + 2L^2 \E \| x_0^s - x^* \|^2
\end{eqnarray} 
\end{lemma}

\begin{proof}
Since $v_t^s$ is $\Psi_i$-sparse, we have
\begin{eqnarray} 
     \left\| v_t^s - D_{i} \nabla \F(x^*) \right\|^2 
    &=& \left\| v_t^s - D \nabla \F(x^*) \right\|_{(i)}^2 \notag \\
    &=& \left\|  \nabla \F_i (\hat{x}_t^s) - \nabla \F_i(x_0^s) + D \nabla \F(x_0^s) - D \nabla \F(x^*) \right\|_{(i)}^2 \notag \\
    &=& \left\| \nabla \F_i (\hat{x}_t^s) - \nabla \F_i(x^*) + D \nabla \F(x_0^s) - D \nabla \F(x^*) - (\nabla \F_i (x_0^s) - \nabla \F_i(x^*)) \right\|_{(i)}^2 \notag \\
    &\le& 2 \left\| \nabla \F_i (\hat{x}_t^s) - \nabla \F_i(x^*) \right\|_{(i)}^2  + 2 \left\| D \nabla \F(x_0^s) - D \nabla \F(x^*) - (\nabla \F_i (x_0^s) - \nabla \F_i(x^*)) \right\|_{(i)}^2
\end{eqnarray} 
According to the support set of $\nabla \F_i$ and Lemma \ref{lm2}, we get
\begin{eqnarray} 
 \E \left\| \nabla \F_i (\hat{x}_t^s) - \nabla \F_i(x^*) \right\|_{(i)}^2 \le 2L \E B_\F(\hat{x}_t^s, x^*)
\end{eqnarray} 
On the other hand,
\begin{eqnarray} 
    && \E \left\| D \nabla \F(x_0^s) - D \nabla \F(x^*) - (\nabla \F_i (x_0^s) - \nabla \F_i(x^*)) \right\|_{(i)}^2 \notag \\
    &=& \E \left\| D \nabla \F(x_0^s) - D \nabla \F(x^*) \right\|_{(i)}^2 + \E \left\| \nabla \F_i (x_0^s) - \nabla \F_i(x^*) \right\|_{(i)}^2 \notag \\
    &&  - 2\E \langle D \nabla \F(x_0^s) - D \nabla \F(x^*), \nabla \F_i (x_0^s) - \nabla \F_i(x^*) \rangle_{(i)} \notag \\
    &=& \E \left\| \nabla \F_i (x_0^s) - \nabla \F_i(x^*) \right\|^2 + \E \left\langle D_i \nabla \F(x_0^s) - D_i \nabla \F(x^*), D \nabla \F(x_0^s) - D \nabla \F(x^*) \right\rangle \notag \\
    &&  - 2\E \langle D \nabla \F(x_0^s) - D \nabla \F(x^*), \nabla \F_i (x_0^s) - \nabla \F_i(x^*) \rangle \notag \\
    &=& \E \left\| \nabla \F_i (x_0^s) - \nabla \F_i(x^*) \right\|^2 + \E \left\langle \nabla \F(x_0^s) - \nabla \F(x^*), D \nabla \F(x_0^s) - D \nabla \F(x^*) \right\rangle \notag \\
    &&  - 2\E \langle D \nabla \F(x_0^s) - D \nabla \F(x^*), \nabla \F(x_0^s) - \nabla \F(x^*) \rangle \notag \\
    &\le& \E \left\| \nabla \F_i (x_0^s) - \nabla \F_i(x^*) \right\|^2 \\
    &\le& L^2 \E \| x_0^s - x^* \|^2,
\end{eqnarray} 
where the second equality is derived by the definition of $D_i$ and the support of $\nabla \F_i$. In the third equality, we take expectation on $i$ and use $\E D_i = I_p$. In the first inequality, we use the fact that $D$ is a diagonal matrix with non-negative entries and hence $\left\langle \nabla \F(x_0^s) - \nabla \F(x^*), D \nabla \F(x_0^s) - D \nabla \F(x^*) \right\rangle \ge 0$. The last inequality comes from Assumption \ref{assumption:lipschitz}. Combining above inequalities, we complete the proof.
\end{proof}

\subsection{Proof of Theorem \ref{theorem:convergence}}
\begin{proof}
At the $s-1$ iteration, we conduct the elimination step over set $\B_{s-1}$ in the outer loop. Since the eliminated variables are zeroes at the optimal, all the sub-problems have the same optimal solution. We have
\begin{eqnarray} 
\| x_{\B_{s}}^* - x_{\B_{s-1}}^* \|^2 = 0,
\end{eqnarray}
where the norm is conducted on the corresponding coordinate and the eliminated variables in $B_{s}$ are filled with $0$. 

Meanwhile, denote $\tilde{\B_{s}} = \{j\in \B_{s-1} | j\notin \B_{s} \}$, we have 
\begin{eqnarray} 
\| x^0_{\B_{s-1}} - x_{\B_{s-1}}^* \|^2 = \| x^0_{\B_{s}} - x_{\B_{s}}^* \|^2 + \|x^0_{\tilde{\B_{s}}} - x_{\tilde{\B_{s}}}^* \|^2.
\end{eqnarray}
Considering $\|x^0_{\tilde{\B_{s}}} - x_{\tilde{\B_{s}}}^* \|^2\geq 0$, we have 
\begin{eqnarray} 
\| x^0_{\B_{s}} - x_{\B_{s}}^* \|^2 \leq \| x^0_{\B_{s-1}} - x_{\B_{s-1}}^* \|^2 .
\label{screening_step}
\end{eqnarray}
Moreover, according to the iteration in the inner loop, we have
\begin{eqnarray} 
 \|x_{\B_{s}}^{t+1}-x_{\B_{s}}^{*}\|^{2}  &=&\|x_{\B_{s}}^{t}-\eta {g}_{t}^{s}-x_{\B_{s}}^{*}\|^{2} \notag \\
&=&\|x_{\B_{s}}^{t}-x_{\B_{s}}^{*}\|^{2}+\|\eta {g}_{t}^{s}\|^{2}-2 \eta\langle{g}_{t}^{s}, x_{\B_{s}}^{t}-x_{\B_{s}}^{*}\rangle \notag \\
&=&\|x_{\B_{s}}^{t}-x_{\B_{s}}^{*}\|^{2}+\|\eta {g}_{t}^{s}\|^{2}-2 \eta\langle{g}_{t}^{s}, \hat{x}^{t}_{\B_{s}}-x_{\B_{s}}^{*}\rangle+2 \eta\langle{g}_{t}^{s}, \hat{x}^{t}_{\B_{s}}-x_{\B_{s}}^{t}\rangle.
\end{eqnarray} 
Applying Lemma \ref{lm3} to $\B_{s}$, we obtain
\begin{eqnarray} 
\|x_{\B_{s}}^{t+1}-x_{\B_{s}}^{*}\|^{2} &\le& \|x_{\B_{s}}^{t}-x_{\B_{s}}^{*}\|^{2} +2 \eta\langle{g}_{t}^{s}, \hat{x}^{t}_{\B_{s}}-x_{\B_{s}}^{t}\rangle + \eta^2(\beta - 1) \| {g}_{t}^{s} \|^2 \notag \\
&&  + \frac{\eta^2}{\beta} \| {v}_{t}^{s-1} - D_{i, {\B_{s}}} \nabla \F(x_{\B_{s}}^{*}) \|^2  - 2\eta \langle{v}_{t}^{s-1}-D_{i, {\B_{s}}} \nabla \F(x_{\B_{s}}^{*}), \hat{x}^{t}_{\B_{s}}-x_{\B_{s}}^{*}\rangle.
\end{eqnarray} 
Since $\E_i D_{i, {\B_{s}}} = I_{p_s}$, we have
\begin{eqnarray} 
\E \langle{v}_{t}^{s-1}-D_{i, {\B_{s}}} \nabla \F(x_{\B_{s}}^{*}), \hat{x}^{t}_{\B_{s}}-x_{\B_{s}}^{*}\rangle &=& \langle \nabla \F(\hat{x}^{t}_{\B_{s}}) -  \nabla \F(x_{\B_{s}}^{*}), \hat{x}^{t}_{\B_{s}}-x_{\B_{s}}^{*}\rangle \notag \\
& \ge& \frac{\mu}{2} \| \hat{x}^{t}_{\B_{s}}-x_{\B_{s}}^{*} \|^2 + B_\F(\hat{x}^{t}_{\B_{s}}, x_{\B_{s}}^{*}),
\end{eqnarray} 
where the inequality is obtained by applying Lemma \ref{lm1} to the sub-problem $\PP_s$. 

Combining above two inequalities we get
\begin{eqnarray} 
\|x_{\B_{s}}^{t+1}-x_{\B_{s}}^{*}\|^{2} &\le& \|x_{\B_{s}}^{t}-x_{\B_{s}}^{*}\|^{2} +2 \eta\langle{g}_{t}^{s}, \hat{x}^{t}_{\B_{s}}-x_{\B_{s}}^{t}\rangle + \eta^2(\beta - 1) \| {g}_{t}^{s}\|^2 \notag \\
&&  + \frac{\eta^2}{\beta} \| {v}_{t}^{s-1} - D_{i, {\B_{s}}} \nabla \F(x_{\B_{s}}^{*}) \|^2 - \eta\mu \| \hat{x}^{t}_{\B_{s}}-x_{\B_{s}}^{*} \|^2 - 2\eta B_\F(\hat{x}^{t}_{\B_{s}}, x_{\B_{s}}^{*}) \notag \\
&\le& (1 - \frac{\eta\mu}{2}) \|x_{\B_{s}}^{t}-x_{\B_{s}}^{*}\|^{2} +2 \eta\langle{g}_{t}^{s}, \hat{x}^{t}_{\B_{s}}-x_{\B_{s}}^{t}\rangle + \eta^2(\beta - 1) \| {g}_{t}^{s} \|^2 \notag \\
&&  + \frac{\eta^2}{\beta} \| {v}_{t}^{s-1} - D_{i, {\B_{s}}} \nabla \F(x_{\B_{s}}^{*}) \|^2   + \eta\mu \| \hat{x}^{t}_{\B_{s}}-x_{\B_{s}}^{t} \|^2 - 2\eta B_\F(\hat{x}^{t}_{\B_{s}}, x_{\B_{s}}^{*}),
\end{eqnarray} 
where in the second inequality we use $\|a+b\|^2 \le 2\|a\|^2 + 2\|b\|^2$. 

By considering Lemma \ref{lm5} and Lemma \ref{lm6} with $\B_{s}$, we have
\begin{eqnarray} 
  \E \|x_{\B_{s}}^{t+1}-x_{\B_{s}}^{*}\|^{2} 
    &\le& (1 - \frac{\eta\mu}{2}) \E \| x_{\B_{s}}^{t} - x_{\B_{s}}^{*} \|^{2} + \frac{2L^2\eta^2}{\beta} \E \| x^0_{\B_{s}} - x_{\B_{s}}^{*} \|^2 + \frac{4L\eta^2}{\beta} \E B_\F(\hat{x}^{t}_{\B_{s}}, x_{\B_{s}}^{*}) \notag \\
&&  - 2\eta \E B_\F(\hat{x}^{t}_{\B_{s}}, x_{\B_{s}}^{*}) + \eta^2 (\beta - 1 + \sqrt{\Delta} \tau) \E \left\| {g}_{t}^{s} \right\|^2 \notag \\
    &&   + (\eta^{3} \mu (1 + \sqrt{\Delta} \tau) + \eta^2 \sqrt{\Delta}) \sum_{u=(t-\tau)_{+}}^{t-1} \E \left\| {g}_{u}^{s} \right\|^{2}.
\end{eqnarray} 
Because $4L\eta \le 1$, $\sqrt{\Delta} \tau \le \frac{1}{10}$ and $B_\F(\cdot, \cdot)$ is non-negative, let $\beta = \frac{1}{2}$, we have
\begin{eqnarray} 
 \E \|x_{\B_{s}}^{t+1}-x_{\B_{s}}^{*}\|^{2} &\le& (1 - \frac{\eta\mu}{2}) \E \| x_{\B_{s}}^{t} - x_{\B_{s}}^{*} \|^{2} + 4L^2\eta^2 \E \| x^0_{\B_{s}} - x_{\B_{s}}^{*} \|^2 - \frac{2\eta^2}{5} \E \left\| {g}_{t}^{s} \right\|^2 \notag \\
    &&  + (2 \eta^3 \mu + \eta^2 \sqrt{\Delta}) \sum_{u=(t-\tau)_{+}}^{t-1} \E \left\| {g}_{u}^{s} \right\|^{2}.
\end{eqnarray} 
Applying recursion to above inequality and we can obtain
\begin{eqnarray} 
\label{recursion}
 \E \|x^K_{\B_{s}}-x_{\B_{s}}^{*}\|^{2} \le \left((1 - \frac{\eta\mu}{2})^K + 8\kappa L\eta\right) \E \| x^0_{\B_{s}} - x_{\B_{s}}^{*} \|^{2} - \sum_{u=0}^{K-1} r_u^s \E \| {g}_{u}^{s} \|^2,
\end{eqnarray} 
where
\begin{eqnarray} 
    r_u^s = (1 - \frac{\eta\mu}{2})^{K-1-u} \Big{(} \frac{2\eta^2}{5} - (2 \eta^3 \mu + \eta^2 \sqrt{\Delta}) \sum_{m=0}^{\tau-1} (1 - \frac{\eta\mu}{2})^{-m} \Big{)}.
\end{eqnarray} 
Let $h(x) = \log(1 + 2ax) - x\log(1 + a)$ for some positive $a$. We can see $h(0) = 0$ and $h'(x) = \frac{2a}{1 + 2ax} - \log(1+a)$. If $ax \le \frac{1}{2}$, then we have $h'(x) \ge a - \log(1+a) \ge 0$. Therefore, when $ax \le \frac{1}{2}$, we always have $h(x) \ge 0$ and $(1 + a)^x \le 1 + 2ax$. Substitute $a$ with $\frac{\eta\mu}{2 - \eta\mu}$. Since $\eta\mu\tau \le \frac{1}{10}$, we have $\frac{\eta\mu\tau}{2 - \eta\mu} \le \frac{1}{2}$. Then we can estimate the lower bound of $r_u^s$:
\begin{eqnarray} 
    r_u^s \ge (1 - \frac{\eta\mu}{2})^{K-1-u} \Big{(} \frac{2\eta^2}{5} - (2 \eta^3 \mu + \eta^2 \sqrt{\Delta}) (\tau + \frac{\eta\mu\tau^2}{2 - \eta\mu}) \Big{)} \ge (1 - \frac{\eta\mu}{2})^{K-1-u} \frac{7\eta^2}{100},
\end{eqnarray} 
where we also use $\eta\mu\tau \le \frac{1}{10}$ and $\sqrt{\Delta}\tau \le \frac{1}{10}$. As $r_u^s$ is positive, we can drop the last term in Eq. (\ref{recursion}). 

As $\eta\mu \le \frac{1}{2}$, it is easy to check that $\log(1 + \frac{\eta\mu}{2 - \eta\mu}) \ge \frac{\eta\mu}{2(2 - \eta\mu)}$. When $K = \frac{4\log 3}{\eta\mu}$, we have $(1 - \frac{\eta\mu}{2})^K \le \frac{1}{3}$. Since $\eta \le \frac{1}{24\kappa L}$, we can prove
\begin{eqnarray} 
 \E \|x^K_{\B_{s}}-x_{\B_{s}}^{*}\|^{2} \le \frac{2}{3} \E \| x^0_{\B_{s}} - x_{\B_{s}}^{*} \|^{2}.
\end{eqnarray} 
Note $x^K_{\B_{s}} = x_{\B_{s}} $ and $x^0_{\B_{s-1}} = x_{\B_{s-1}} $, combining (\ref{screening_step}), we have 
\begin{eqnarray} 
 \E \|x_{\B_{s}}-x_{\B_{s}}^{*}\|^{2} \le \frac{2}{3} \E \| x_{\B_{s-1}} - x_{\B_{s-1}}^{*} \|^{2}.
\end{eqnarray} 
Apply recursion to the above inequality, note $x_{\B_{0}} = x_0$ and $x_{\B_{0}}^{*} = x^*$ we can obtain
\begin{eqnarray} 
 \E \|x_{\B_{S}}-x_{\B_{S}}^{*}\|^{2} \le (\frac{2}{3})^S \E \| x_{0} - x^{*} \|^{2},
\end{eqnarray} 
which finishes the proof.
\end{proof}

\subsection{Proof of Theorem \ref{theorem:screening}}
\begin{proof}
Based on Theorem \ref{theorem:convergence}, we know our DDSS method has a converging sequence $\{ x_{\B_{s}}\}$. As DDSS algorithm converges, because of the strong duality, the dual $y^s$ and the intermediate duality gap $\PP(x_{\B_{s}})-D(y^s)$ also converges. For any given $\epsilon$,  $\exists S_{0}$ such that $\forall s \geq S_{0}$, we have 
\begin{eqnarray} \label{converge1}
    \| y^s - y^* \|_{2} \leq  \epsilon,
\end{eqnarray}
and
\begin{eqnarray} \label{converge2}
 \sqrt{2L(\PP(x_{\B_{s}})-D(y^s))}\leq  \epsilon. 
\end{eqnarray}
almost surely.

For any $j \notin \B^{*} $, we have 
\begin{eqnarray} 
&& \Omega^D_j(A^\top_{j} y^s) + \Omega^D_j(A_{j}) \sqrt{2L( \PP(x_{\B_{s}})-D(y^s))} \nonumber \\
& \leq &  \Omega^D_j(A^\top_{j} (y^s- y^*)) + \Omega^D_j(A^\top_{j}  y^*)  + \Omega^D_j(A_{j})  \sqrt{2L( \PP(x_{\B_{s}})-D(y^s))} \\
& \leq &  2 \Omega^D_j(A_{j}) \epsilon + \Omega^D_j(A^\top_{j}  y^*)  \nonumber
\end{eqnarray}
where the first inequality comes from the triangle inequality and the second inequality comes from (\ref{converge1}) and (\ref{converge2}). 

Hence, if we choose  
\begin{eqnarray} \label{eq:epsilon}
    \epsilon < \frac{ n \lambda-\Omega^D_j(A^\top_{j}  y^*) }{2 \Omega^D_j(A_{j})},
\end{eqnarray}
we can ensure the screening test $\Omega^D_j(A^\top_{j} y^s) + \Omega^D_j(A_{j}) \sqrt{2L( \PP(x_{\B_{s}})-D(y^s))} < n \lambda $ holds for $j$, which means variable block $j$ is eliminated at most at this iteration. In (\ref{eq:epsilon}), since  $j \notin \B^{*}$, it is easily to verify that $n \lambda-\Omega^D_j(A^\top_{j}  y^*) > 0 $.  This finishes the proof.
\end{proof}

\end{document}